\documentclass[table]{article}

\usepackage{arxiv}

\usepackage[utf8]{inputenc} 
\usepackage[T1]{fontenc}    
\usepackage{hyperref}       
\usepackage{url}            
\usepackage{booktabs}       
\usepackage{amsfonts}       
\usepackage{nicefrac}       
\usepackage{microtype}      

\usepackage{lipsum}         
\usepackage{graphicx}
\usepackage{natbib}
\usepackage{doi}

\usepackage{amsmath,amssymb,amsthm}
\usepackage{cleveref}       
\usepackage{multirow} 
\usepackage{graphicx} 
\usepackage{tikz} 
\usetikzlibrary{positioning, calc, fit} 
\usepackage{float} 
\usepackage{geometry} 
\usepackage{xcolor}
\usepackage{pifont} 

\usepackage{subcaption}

\newtheorem{lemma}{Lemma}
\newtheorem{theorem}{Theorem}

\newtheorem{proposition}{Proposition}

\newcommand{\bestcell}[1]{\cellcolor{gray!30}\textbf{#1}}
\newcommand{\secondcell}[1]{\cellcolor{gray!18}\textbf{#1}}
\newcommand{\thirdcell}[1]{\cellcolor{gray!10}\textbf{#1}}

\title{Affordance-First Decomposition for Continual Learning in Video–Language Understanding}

\author{
Mengzhu Xu$^{\dagger}$ \\
University of Sydney
\And
Hanzhi Liu$^{\dagger}$ \\
University of California, Santa Barbara
\And
Ningkang Peng \\
Nanjing Normal University
\And
Qianyu Chen \\
Nanyang Technological University
\And
Canran Xiao$^{*}$ \\
Shenzhen Campus of Sun Yat-sen University
}

\renewcommand{\shorttitle}{\textit{arXiv} Template}

\hypersetup{
pdftitle={A template for the arxiv style},
pdfsubject={q-bio.NC, q-bio.QM},
pdfauthor={David S.~Hippocampus, Elias D.~Striatum},
pdfkeywords={First keyword, Second keyword, More},
}

\begin{document}
\maketitle

\begin{abstract}
Continual learning for video--language understanding is increasingly important as models face non-stationary data, domains, and query styles, yet prevailing solutions blur what should stay stable versus what should adapt, rely on static routing/capacity, or require replaying past videos. 
We aim to explicitly specify where stability lives and where plasticity should be focused under realistic memory and privacy constraints. 
We introduce Affordance-First Decomposition (AFD): videos are mapped to slowly varying affordance tokens that form a shared, time-aligned substrate, while a lightweight, query-routed, conflict-aware scheduler concentrates adaptation and grows capacity only when needed. The substrate is stabilized via weak alignment and teacher consistency, and training uses question-only replay. 
AFD achieves state-of-the-art across protocols: 51.6\% average accuracy with $-1.8\%$ forgetting on domain-incremental VideoQA, ViLCo R@1@0.5 of 29.6\% (MQ) and 20.7\% (NLQ) with 18.4\% stAP@0.25 (VQ), and 39.5\% accuracy with $-1.6\%$ forgetting on time-incremental iVQA.
Overall, AFD offers an explicit, interpretable split between a stable interaction-centered substrate and targeted adaptation.
\end{abstract}


\section{Introduction}
\label{sec:intro}

\begin{figure}[htb]
	\centering
	\includegraphics[width=0.6\linewidth]{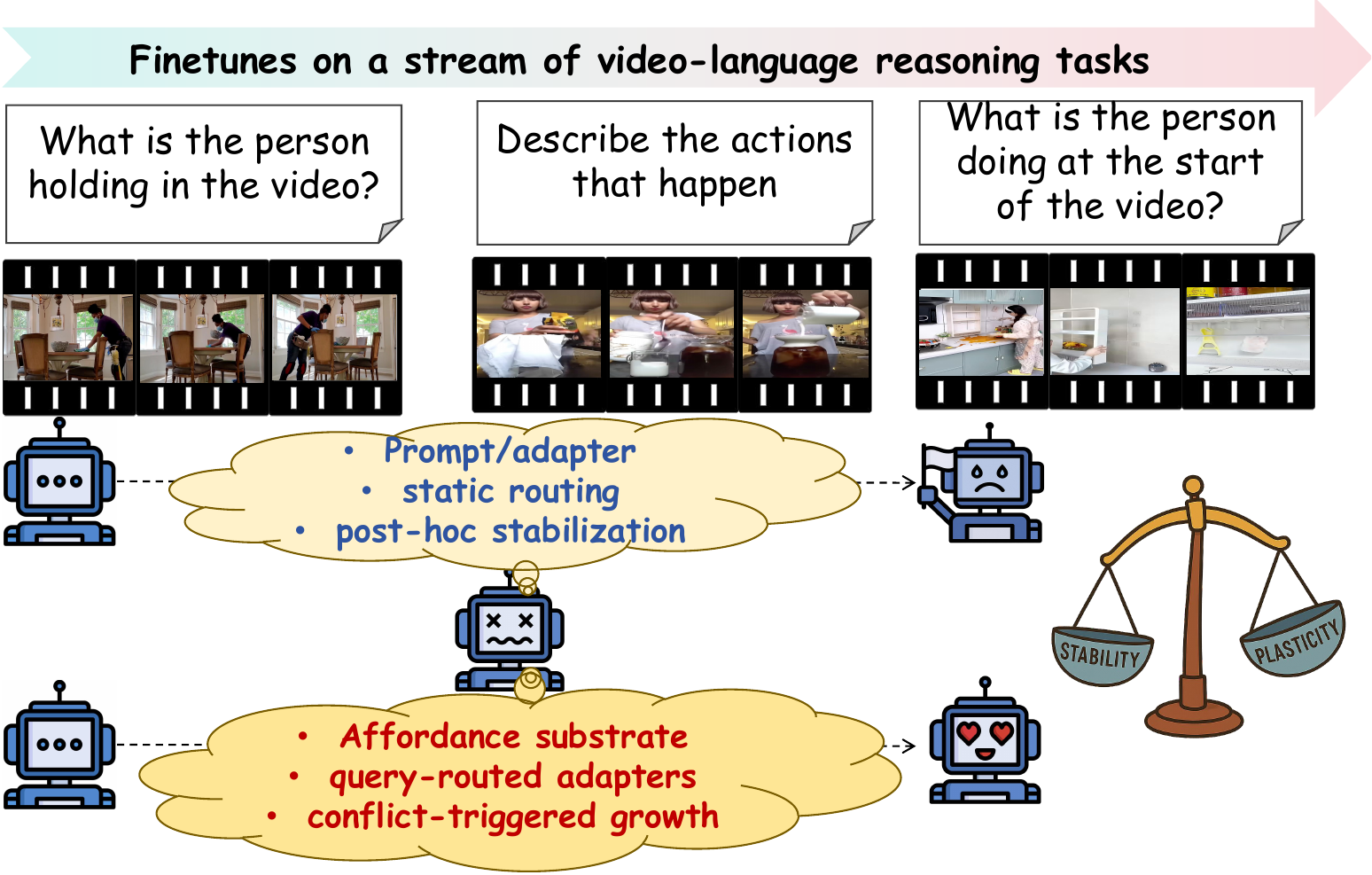}
	\caption{Under a stream of video–language reasoning tasks, existing methods rely on prompt/adapter add-ons with static routing and post-hoc stabilization, leaving the stability–plasticity trade-off implicit. AFD instead anchors evidence in a slowly varying affordance substrate and applies query-routed, conflict-triggered adapter updates, explicitly separating stability from plasticity.}
	\label{fig:teaser}
\end{figure}

Video understanding~\cite{buch2022revisiting} underpins assistive analytics, retrieval, and embodied agents, where models must parse long, multi-event sequences and align visual dynamics with natural language~\cite{lavee2009understanding}. Yet real deployments seldom operate in a stationary world: data, domains, and query styles evolve over time~\cite{Bisecle25}. \emph{Continual video–language reasoning} is therefore central to assistants that must keep learning from non-stationary streams while answering open-form queries, and step reasoning in videos. 
Recent studies underscore both the opportunity and difficulty: long egocentric protocols mix heterogeneous tasks and language forms~\cite{ViLCo24}, and time-continual pretraining shows that naïve fine-tuning quickly drifts while replay-heavy strategies raise cost and privacy concerns~\cite{ticclip24,FoMo}. Robust solutions must retain prior skills, acquire new ones, and remain efficient without relying on storing old videos~\cite{VideoLLaMB25}.

Despite rapid progress, two shortcomings persist: \emph{(i) Objectives for stability vs.\ adaptability are under-specified.} Existing lines either specialize with prompts/adapters or preserve geometry via distillation/topology constraints, but rarely articulate \emph{which structures should remain task-invariant and which should adapt} along a stream—making stability largely incidental and hard to diagnose~\cite{ColPro24,Dam25,Bisecle25,zscl23,ctp23}. \emph{(ii) Plasticity is budgeted heuristically.} Capacity and routing are commonly fixed or task-indexed, while interference is mitigated post hoc by merging or global regularizers. Few approaches use online signals to decide \emph{when/where} to change~\cite{L2P,CODA-Prompt,ddas24,DIKI24,clmoe25}.

This paper asks: \emph{Can continual video–language learning be anchored in a slowly varying, interaction-centered substrate that separates stability from adaptation?} As shown in Fig.~\ref{fig:teaser}, we adopt an \emph{affordance-first decomposition(AFD)}: videos are mapped into affordance evidence, while adaptation is concentrated in a lightweight, query-routed reasoning module. The affordance evidence provides time-aligned, reusable signals that remain stable across tasks. The routed module focuses plastic updates only where conflict arises, preserving past capabilities without broad parameter drift. 

Our contributions are as follows:
 \begin{itemize}
  \item We introduce an \emph{affordance-first decomposition} that separates a \emph{slowly varying} shared substrate from a \emph{plastic} routed scheduler, clarifying where stability vs.\ adaptability should live in continual video–language learning.
  \item We operationalize conflict-aware adaptation by query-conditioned per-layer routing and selective capacity growth, and adopt a question-only replay strategy that is privacy- and memory-friendly.
  \item Across ViLCo\cite{ViLCo24} and standard VideoQA suites, our approach achieves state-of-the-art results with substantially lower forgetting compared to strong baselines and recent SOTAs, while remaining order-robust and compute-efficient.
\end{itemize}

\section{Related Work}
\label{sec:related}

\subsection{Continual video–language learning.}
Early multimodal CL revealed strong order sensitivity and forgetting in linguistically structured VQA and captioning~\cite{VQA_CL,RATT20}. ViLCo-Bench later standardized long-video continual protocols and shifted evaluation toward open-form reasoning~\cite{ViLCo24}. Recent SOTAs adapt LLM/VLM backbones via prompting or adapters—ColPro injects collaborative prompts~\cite{ColPro24}, DAM merges dataset-wise adapters at inference~\cite{Dam25}, and Bisecle couples binding with separation to reduce interference~\cite{Bisecle25}. Yet these approaches often specialize by dataset or prompt banks and model interference only implicitly. AFD departs by separating a slowly varying affordance head from a query-routed, conflict-aware scheduler and by relying on question-only replay.

\subsection{Parameter-efficient routing and prompting.}
Prompt/adapter methods enable rehearsal-free selectivity for CLIP-style models~\cite{coop22,L2P,CODA-Prompt}, while adapter/MoE variants improve transfer/retention via selective gating or consolidation~\cite{ddas24,DIKI24,SD24,clap4clip24,RAIL24,cclip25,clmoe25}. However, capacity is typically fixed (e.g., prompt count or LoRA rank), routers are often task/domain-driven, and the stabilized representation is not made explicit. Our scheduler instead performs per-layer query-conditioned routing over LoRA experts and grows rank only when measured conflict exceeds a threshold, concentrating plasticity while bounding capacity.

\subsection{Continual multimodal learning.}
Geometry-preserving CL for VLMs aligns cross-/intra-modal similarity or momentum topologies to protect zero-shot ability~\cite{modx23,zscl23,ctp23}, pragmatic training shows that warm-start plus replay approaches full re-training at far lower cost~\cite{ticclip24}. Privacy-aware/data-free directions use structured or synthetic replay and rectify teacher noise~\cite{constructvl23,vqacl23,gift25,dkr24,quad25}. These streams still blur what should remain shared and where to place plasticity. AFD contributes an interpretable, slow-varying affordance substrate.

\begin{figure}[thb]
	\centering
	\includegraphics[width=0.9\linewidth]{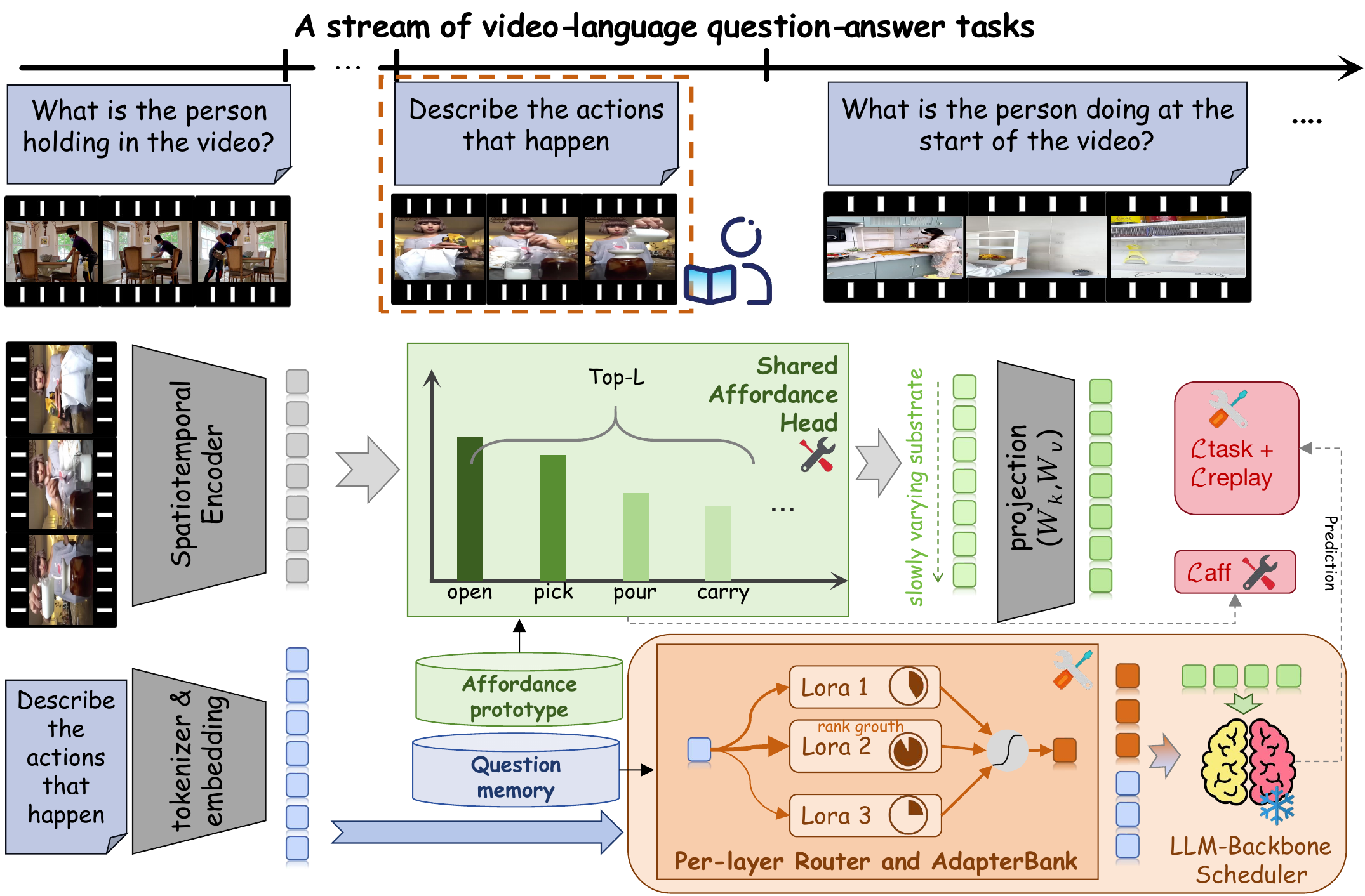}
	\caption{Overview of the proposed Affordance-First Decomposition (AFD) framework for continual video–language question answering. A stream of video–language tasks arrives over time, each video is encoded and mapped by a shared affordance head into slowly varying affordance tokens and prototypes, while questions are embedded and stored for replay to route per-layer LoRA adapters in the LLM-backbone scheduler. Stability loss $\mathcal{L}_{\text{aff}}$ acts only on the affordance head, whereas task and replay losses $(\mathcal{L}_{\text{task}}+\mathcal{L}_{\text{replay}})$ act only on the routed adapters, explicitly separating a stable affordance substrate from a plastic reasoning module.}
	\label{fig:pipeline}
\end{figure}

\section{Method}
\label{sec:method}
We address continual video–language reasoning where tasks arrive as a stream and both domains and query formats evolve over time. Our approach is \textbf{Affordance-First Decomposition}(AFD). A shared head converts a video into temporally grounded affordance tokens. A plastic LLM-backbone scheduler consumes the query tokens together with these affordance tokens and performs event-level reasoning through per-layer routed low-rank adapters. Stability acts only on the shared head. Plasticity and task specialization are absorbed by the LLM scheduler. Two compact memories enable practical rehearsal. Figure~\ref{fig:pipeline} presents the pipeline.

\subsection{Problem setup}
We work in a streaming setting where tasks arrive over time. Each task provides labeled triples $(V,q,y)$ and an unlabeled clip pool for stability. A video $V=\{F_t\}_{t=1}^{T}$ is the visual input. A query $q$ is free text. The target $y$ is either an open answer or a temporal span or a step sequence.

The model has two parts that play different roles. The \emph{shared affordance head} $h_{\psi}$ converts the video into continuous affordance tokens that are linearly projected into the LLM hidden space to form keys and values $(K,V)$. The \emph{LLM-backbone scheduler} $g_{\phi}^{\text{LLM}}$ embeds the query with the same LLM to obtain $U$ and attends to $(K,V)$ to produce the task-appropriate prediction. The overall predictor is
\begin{equation}
\label{eq:factorization}
f_{\Theta}(V,q) = g_{\phi}^{\text{LLM}}\!\big(U, K, V\big)
\end{equation}
where $U = E_{\text{LLM}}[\mathrm{Tok}(q)]$ , $(K,V) = \Pi(h_{\psi}(V))$.
We maintain two small memories that serve training only. $\mathcal{M}_Q$ stores diverse past questions for replay distillation. $\mathcal{M}_A$ stores affordance prototypes for diagnostics. By design, stability constraints are applied to the shared head $h_{\psi}$ and task plasticity is absorbed by the scheduler $g_{\phi}^{\text{LLM}}$.

\subsection{Architecture}
\paragraph{Shared affordance head $h_{\psi}$}
Affordances are object–action regularities that vary slowly across domains and tasks. A stable affordance space reduces gradient conflict for downstream reasoning.

Let $X_{1:T}=\mathrm{Enc}_v(V)$ with $X_t\in\mathbb{R}^{d_v}$. The head produces an affordance distribution
\begin{equation}
\begin{aligned}
z_t &= f_{\text{st}}(X_t), \\
s_t(a) &= \langle w_a, z_t\rangle, \\
P_t(a) &= \mathrm{softmax}_{a\in\mathcal{V}_A}\!\big(s_t(a)/\tau\big).
\end{aligned}
\label{eq:aff-softmax}
\end{equation}
We form a sparse renormalized distribution on the top $L$ categories,
\begin{equation}
q_t(a)=\frac{\mathbf{1}[a\in\text{Top-}L]\cdot P_t(a)}{\sum_{a'\in \text{Top-}L}P_t(a')},
\end{equation}
then build a continuous token with an embedding table $E_A\in\mathbb{R}^{|\mathcal{V}_A|\times d_a}$,
\begin{equation}
A_t=\sum_{a\in\mathcal{V}_A} q_t(a)\,E_A[a]\in\mathbb{R}^{d_a}.
\label{eq:aff-token}
\end{equation}
To interface with the LLM hidden space we project
\begin{equation}
K_t = W_K A_t, \quad V_t = W_V A_t,
\label{eq:proj}
\end{equation}
where stack $K=\mathrm{stack}_t(K_t)$ and $V=\mathrm{stack}_t(V_t)$. $W_K, W_V \in \mathbb{R}^{d_{\text{model}} \times d_a}$. 

\paragraph{LLM tokenizer and text interface}
The scheduler is an LLM. Queries must be embedded in its native space for unified conditioning and generation.
Let $\mathrm{Tok}$ be the LLM tokenizer with vocabulary $\mathcal{V}_{\text{LLM}}$ and $E_{\text{LLM}}\in\mathbb{R}^{|\mathcal{V}_{\text{LLM}}|\times d_q}$ its input embedding matrix. We compute
\begin{equation}
    \begin{aligned}
    [w_1,\dots,w_L] &= \mathrm{Tok}(q), \\
    U = [u_1,\dots,u_L]\ &\text{with}\ u_\ell = E_{\text{LLM}}[w_\ell], \\
    u &= \mathrm{Pool}(U).
    \end{aligned}
\label{eq:text}
\end{equation}
where $U\in\mathbb{R}^{L\times d_q}$ are token embeddings. $u\in\mathbb{R}^{d_q}$ is a pooled query state used by the router.

\paragraph{Per-layer routing and adapter injection}
Heterogeneous tasks require different reasoning skills. Instance-wise routing focuses plastic capacity and reduces interference.
At each adapterized LLM layer $\ell\in\mathcal{S}$ the router computes mixture weights
\begin{equation}
\alpha^{(\ell)}=\mathrm{softmax}\!\big(W_r^{(\ell)}\,u\big)\in\Delta^{m-1},
\label{eq:router}
\end{equation}
and injects a mixture of LoRA experts into the linear map $W^{(\ell)}$
\begin{equation}
\widetilde{W}^{(\ell)} = W^{(\ell)} + \sum_{j=1}^{m}\alpha_{j}^{(\ell)}\,\frac{B_{j}^{(\ell)}A_{j}^{(\ell)}}{s_{j}^{(\ell)}},
\label{eq:lora}
\end{equation}
where $A_{j}^{(\ell)}\!\in\!\mathbb{R}^{r_{j}^{(\ell)}\times d_{\text{in}}}$ and $B_{j}^{(\ell)}\!\in\!\mathbb{R}^{d_{\text{out}}\times r_{j}^{(\ell)}}$ are low-rank factors with rank $r_{j}^{(\ell)}$ and $s_{j}^{(\ell)}$ is a scale.

We measure conflict by a clamped negative cosine with numerical stabilization
\begin{equation}
c_j^{(k)}=\left[
-\frac{\langle g_j^{(k)},\,\bar g_j^{(1:k-1)}\rangle}{\|g_j^{(k)}\|_2\,\|\bar g_j^{(1:k-1)}\|_2+\varepsilon}
\right]_+,\qquad \varepsilon>0.
\label{eq:conflict}
\end{equation}
The LoRA rank grows discretely by the excess conflict above a threshold and is capped
\begin{equation}
\begin{aligned}
\Delta r_j^{(k)} &= \min\!\Big\{r_{\max} - r_j^{(k-1)},\ \big\lfloor \gamma\,(c_j^{(k)} - \tau_c)_+\big\rfloor\Big\}, \\
r_j^{(k)} &= r_j^{(k-1)} + \Delta r_j^{(k)}.
\end{aligned}
\label{eq:rank-update}
\end{equation}
Here $\tau_c\in[0,1)$ is the threshold, $\gamma>0$ is a gain, and $(\cdot)_+=\max\{\cdot,0\}$. Initialization of new columns can follow a truncated SVD of the projected gradient with Tikhonov-regularized inverses, detailed in the Supplementary Materials.

\paragraph{LLM-backbone scheduler with affordance cross-attention}
The LLM composes language evidence with affordance evidence and outputs the final reasoning result while insulating the shared head from frequent changes.

At layers $\ell\in\mathcal{S}$ the LLM attends to affordances
\begin{equation}
\begin{aligned}
Q &= U W_Q, \\
H &= \mathrm{Attn}(Q, K, V), \\
r &= \mathrm{Pool}(H, U).
\end{aligned}
\label{eq:cross}
\end{equation}
and task heads support three query formats with a unified supervision
\begin{equation} \small
\begin{aligned}
\mathcal{L}_{\text{task}}
&= \mathbb{I}_{\text{gen}}\!\left[-\sum_{m}\log p(y_m \mid y_{<m}, U, K, V)\right] \\
&\quad + \mathbb{I}_{\text{span}}\!\left[-\log p_s(t_s) - \log p_e(t_e) + \lambda_{\mathrm{u}}\big(1 - \mathrm{tIoU}\big)\right] \\
&\quad + \mathbb{I}_{\text{step}}\!\left[-\sum_{m}\log p(\pi_m \mid \pi_{<m}, U, K, V)\right].
\end{aligned}
\label{eq:task}
\end{equation}
Here $t_s,t_e\in\{1,\dots,T\}$ and $\mathrm{tIoU}$\cite{lan2023survey} is computed on discrete frame intervals. Each sample activates exactly one head indicated by the selector $\mathbb{I}_{\text{gen}}$ or $\mathbb{I}_{\text{span}}$ or $\mathbb{I}_{\text{step}}$.

\subsection{Training objective}
\paragraph{Affordance stability on $h_{\psi}$}
We blend weak alignment with teacher consistency
\begin{equation}
\begin{aligned}
\mathcal{L}_{\text{aff}}
&= \beta\left[-\sum_{\ell}\log\Big(\sum_{t\in S_\ell}\sum_{a\in\mathcal{C}_\ell}P_t(a)\Big)\right] \\
&\quad + (1-\beta)\,\frac{1}{T}\sum_{t=1}^{T}\mathrm{KL}\!\left(\bar{P}_t\,\|\,P_t\right),
\end{aligned}
\label{eq:aff}
\end{equation}
where $S_\ell$ is an ASR span with verb candidates $\mathcal{C}_\ell$. $P_t(\cdot)$ is the current affordance distribution and $\bar{P}_t(\cdot)$ is the frozen teacher from the previous task. The scalar $\beta\in[0,1]$ balances the two terms. Gradients of $\mathcal{L}_{\text{aff}}$ update $\psi$ only.

\paragraph{Question-only replay distillation on $g_{\phi}^{\text{LLM}}$}
We store diverse past questions and distill on current clips with temperature $T_{\mathrm{kd}}>0$ and optional confidence masking $\rho\in(0,1)$
\begin{equation}
\begin{aligned}
\mathcal{L}_{\text{replay}}
&= \mathbb{E}_{q^{(u)},V}\ \mathrm{KL}\!\Big(\bar{p}_{T}(\cdot\mid V,q)\,\|\,p_{T}(\cdot\mid V,q)\Big), \\
&\quad\text{where } p_{T} = \mathrm{softmax}(z/T_{\mathrm{kd}}).
\end{aligned}
\label{eq:replay}
\end{equation}
and we include only pairs whose teacher maximum probability exceeds $\rho$ to suppress noisy supervision. Here, the expectation is over $q^{(u)} \in \mathcal{M}_Q$ and $V$, and $p_{T} = \mathrm{softmax}(z/T_{\mathrm{kd}})$ denotes the temperature-scaled probability distribution. Gradients of $\mathcal{L}_{\text{replay}}$ update $\phi$.

\paragraph{Full objective}
At task $k$ we minimize the three-term objective
\begin{equation}
\mathcal{L}^{(k)}=\mathcal{L}_{\text{task}}^{(k)}+\lambda_{\text{aff}}\mathcal{L}_{\text{aff}}^{(k)}+\lambda_{\text{rep}}\mathcal{L}_{\text{replay}}^{(k)}
\label{eq:full}
\end{equation}
with positive scalars $\lambda_{\text{aff}}$ and $\lambda_{\text{rep}}$. Gradients of $\mathcal{L}_{\text{task}}$ and $\mathcal{L}_{\text{replay}}$ update the LLM scheduler with routed adapters and rank growth.
\section{Experimental Results}

\subsection{Experimental Setup}
\label{sec:exp-setup}

\paragraph{Datasets}
\textbf{(i) ViLCo-Bench.}
To evaluate continual video--language reasoning across heterogeneous tasks, we adopt \textbf{ViLCo-Bench} with its three continual tracks built from Ego4D: \emph{Moment Query (MQ)}, \emph{Natural Language Query (NLQ)}, and \emph{Visual Query (VQ)}~\cite{ViLCo24}. We follow its official query-incremental protocols: MQ (5 tasks, 110 actions), NLQ (13 tasks, open-vocabulary queries), and VQ (5 tasks with vision queries).
\noindent\textbf{(ii) Continual VideoQA suites.}
For domain- and time-incremental VideoQA, we adopt the datasets used by recent continual VL methods~\cite{Dam25}. We use the \emph{dataset-incremental} protocol (train adapters sequentially per dataset, evaluate on all test sets) and the \emph{time-incremental} protocol on iVQA by partitioning videos by upload time, as in \cite{Dam25}.
\noindent\textbf{(iii) Complex reasoning VideoQA.}
To probe multi-step reasoning and planning, we include \textbf{CVQA} from VQAGuider~\cite{VQAGuider25} and the \textbf{11-VideoQA} benchmark set from LTR~\cite{LTR25} (we report the overlapped common sets with our training budget). These tracks are not used for pure continual metrics. Instead, we report zero-shot/finetuned generalization and sequence-trained robustness.
\noindent\textbf{(iv) Long-video understanding (stress test).}
To measure long-range temporal robustness of our AFD, we include \textbf{VideoMME} and \textbf{MLVU} used by LongVU~\cite{LongVU25} and the long-context evaluations in VideoLLaMB~\cite{VideoLLaMB25}.

\begin{table*}[t]
\centering
\small
\setlength{\tabcolsep}{6pt}
\caption{\textbf{Domain-Incremental VideoQA on 6 datasets.} Metric: top-1 accuracy (\%). Best, second-best, and third-best cells are shaded in \colorbox{gray!30}{\strut \textbf{dark}} gray, \colorbox{gray!18}{\strut \textbf{medium}} gray, and \colorbox{gray!10}{\strut \textbf{light}} gray, respectively. }
\label{tab:vidqa-main}
\begin{tabular}{lcccccccc}
\toprule
Method & iVQA & MSVD & MSRVTT & LSMDC & ANet & TGIF & Avg.$\uparrow$ & Forget$\downarrow$ \\
\midrule
\multicolumn{9}{l}{\emph{Upper bounds}}\\
Adapters (Multitask) & 39.7 & 56.6 & 46.7 & 62.9 & 42.2 & 67.8 & 52.6 & -- \\
Prompt Tuning (Multitask) & 35.0 & 49.0 & 37.1 & 57.4 & 33.9 & 59.2 & 45.3 & -- \\
\midrule
\multicolumn{9}{l}{\emph{Continual methods}}\\
Zero-Shot & 26.8 & 33.0 & 15.0 & 51.5 & 25.5 & 41.9 & 32.3 & -- \\
Seq-FT & 28.4 & 36.0 & 23.7 & 52.1 & 31.2 & 67.6 & 39.8 & -- \\
EWC~\cite{EWC} & 29.9 & 39.3 & 25.5 & 54.9 & 32.4 & \bestcell{68.5} & 41.6 & $-10.9$ \\
LwF~\cite{LwF} & 28.3 & 38.2 & 25.8 & 56.4 & 33.6 & \secondcell{67.7} & 41.8 & $-10.7$ \\
L2P~\cite{L2P} & 32.8 & 43.3 & 32.1 & 54.8 & 27.2 & 54.4 & 40.8 & $-4.6$ \\
CODA-Prompt~\cite{CODA-Prompt} & 32.9 & 44.8 & 28.7 & 50.7 & 23.9 & 54.7 & 39.6 & $-5.7$ \\
S-Prompts~\cite{S-Prompts} & 31.8 & 45.5 & 30.2 & 54.9 & 27.9 & 56.1 & 41.1 & $-4.2$ \\
MoE (adapters) & 31.7 & 37.1 & 23.9 & 57.7 & 28.9 & 66.8 & 41.0 & $-11.6$ \\
ColPro~\cite{ColPro24} & 35.3 & 49.6 & 36.7 & 58.4 & 32.1 & 61.0 & 45.5 & $-3.9$ \\
LAE~\cite{LAE23} & 36.1 & 50.2 & 37.5 & 58.8 & 32.7 & 61.5 & 46.1 & $-3.4$ \\
Bisecle~\cite{Bisecle25} & \thirdcell{38.9} & \thirdcell{52.1} & \thirdcell{41.3} & \thirdcell{62.1} & \thirdcell{35.4} & 66.3 & \thirdcell{49.4} & \thirdcell{-2.7} \\
DAM~\cite{Dam25} & \secondcell{39.1} & \secondcell{53.6} & \secondcell{42.2} & \secondcell{63.0} & \secondcell{36.3} & 66.8 & \secondcell{50.2} & \secondcell{-2.3} \\
\rowcolor{gray!30!white}
\textbf{AFD (ours)} & \bestcell{40.7} & \bestcell{55.8} & \bestcell{43.7} & \bestcell{63.6} & \bestcell{38.0} & 68.1 & \bestcell{51.6} & \bestcell{-1.8} \\
\bottomrule
\end{tabular}

\vspace{0.5em}
\raggedright \footnotesize \emph{Notes.} The six datasets follow the sequence iVQA $\rightarrow$ MSVD $\rightarrow$ MSRVTT $\rightarrow$ LSMDC $\rightarrow$ ActivityNet (ANet) $\rightarrow$ TGIF. “Adapters/Prompt Tuning (Multitask)” jointly train on all datasets and are shown only as non-continual ceilings.
\end{table*}

\paragraph{Evaluation Metrics}
\textbf{(i)~ViLCo-Bench.}
We follow the official metrics~\cite{ViLCo24}: average recall $\text{R@1}$/$\text{R@5}$ at IoU thresholds (MQ/NLQ), with \emph{Backward Forgetting (BwF)} as the continual metric; for VQ we report \emph{tAP@0.25}, \emph{stAP@0.25}, average recall, and success rate. We also compute the average performance $P$ across tasks as in \cite{ViLCo24}.
\textbf{(ii)~VideoQA.}
For single-answer VideoQA, we report \emph{top-1 accuracy}. For multiple-choice datasets, we use \emph{MC Acc}. For open-ended datasets with standard VQA-style processing we additionally report \emph{EM}/\emph{F1} when applicable .
\textbf{(iii)~Continual learning diagnostics.}
In addition to BwF on ViLCo, we report \emph{Average Accuracy} after each task and the standard \emph{Backward Transfer (BWT)} and \emph{Forgetting} where defined in the compared papers. For time-incremental iVQA, we report per-slice accuracy and average across time slices~\cite{Dam25}.

Please refer to the supplementary materials for baselines and the implementation details of AFD.

\subsection{Main results}
\label{sec:main-results}

\vspace{1mm}
\noindent\textbf{VideoQA.}
Table~\ref{tab:vidqa-main} shows that AFD attains the best average accuracy among continual methods (51.6\%), surpassing DAM by +1.4 points while maintaining the lowest forgetting (\(-1.8\)). Gains are consistent on five of six datasets. 

\begin{table}[t]
\centering
\small
\setlength{\tabcolsep}{3pt}
\caption{\textbf{ViLCo-Bench (Ego4D) under query-incremental protocols.}}
\label{tab:vilco-main}
\resizebox{0.8\linewidth}{!}{%
\begin{tabular}{lccc}
\toprule
Method & MQ R@1@0.5 $\uparrow$ & NLQ R@1@0.5 $\uparrow$ & VQ stAP@0.25 $\uparrow$ \\
\midrule
ViLCo~\cite{ViLCo24} & 21.2 & 12.6 & 13.4 \\
ColPro~\cite{ColPro24} & 26.2 & \thirdcell{17.8} & 15.9 \\
DAM~\cite{Dam25}& \thirdcell{27.1} & 16.9 & \thirdcell{16.5} \\
Bisecle~\cite{Bisecle25} & 26.8 & \secondcell{18.2} & 16.1 \\
\rowcolor{gray!30!white}
\textbf{AFD (ours)} & \bestcell{29.6} & \bestcell{20.7} & \bestcell{18.4} \\
\bottomrule
\end{tabular}%
}
\end{table}

\vspace{1mm}
\noindent\textbf{ViLCo-Bench.}
As summarized in Table~\ref{tab:vilco-main}, AFD achieves the strongest performance on all three query types: +2.5 R@1 on MQ, +2.5 R@1 on NLQ, and +1.9 stAP on VQ over the best competing baseline. 


\begin{table*}[t]
\centering
\small
\setlength{\tabcolsep}{5.2pt}
\caption{\textbf{Additional evaluations.} Left: \emph{time-incremental iVQA} (4 temporal slices by upload time). Right: \emph{complex reasoning} (non-continual) and \emph{long-video stress tests} (non-continual). Best, second-best, and third-best cells are shaded in \colorbox{gray!30}{\strut \textbf{dark}}/\colorbox{gray!18}{\strut \textbf{medium}}/\colorbox{gray!10}{\strut \textbf{light}} gray, respectively.}
\label{tab:aux-evals}

\resizebox{0.8\linewidth}{!}{%
\begin{minipage}{0.64\textwidth}
\centering
\textbf{(a) Time-Incremental iVQA (top-1 \%, higher is better)}
\vspace{0.4em}

\begin{tabular}{lcccccc}
\toprule
Method & S1 & S2 & S3 & S4 & Avg.$\uparrow$ & Forget$\downarrow$ \\
\midrule
Zero-Shot & 26.9 & 26.6 & 26.3 & 26.0 & 26.5 & -- \\
Seq-FT & 29.1 & 27.4 & 27.0 & 26.3 & 27.5 & -- \\
EWC~\cite{EWC} & 31.5 & 30.6 & 30.0 & 29.4 & 30.4 & $-6.7$ \\
LwF~\cite{LwF} & 31.2 & 30.3 & 29.8 & 29.1 & 30.1 & $-6.9$ \\
SMoE~\cite{SMoE25} & 35.3 & 33.6 & 31.5 & 34.2 & 36.9 & $-3.5$ \\
DIKI~\cite{DIKI24} & 31.4 & 28.6 & 30.2 & 31.9 & 34.4 & $-5.3$ \\
DMNSP~\cite{DMNSP25} & 34.0 & 33.6 & 30.6 & 33.9 & 36.4 & $-4.2$ \\
SMoLoRA~\cite{SMoLoR} & 33.2 & 31.0 & 29.4 & 32.5 & 34.1 & $-5.6$ \\
L2P~\cite{L2P} & 34.9 & 34.0 & 33.1 & 32.5 & 33.6 & $-3.9$ \\
ColPro~\cite{ColPro24} & 36.1 & 35.2 & 34.1 & 33.5 & 34.7 & $-3.3$ \\
Bisecle~\cite{Bisecle25} & \thirdcell{39.8} & \thirdcell{37.7} & \thirdcell{36.8} & \thirdcell{36.1} & \thirdcell{37.6} & \thirdcell{$-2.5$} \\
DAM~\cite{Dam25} & \secondcell{40.2} & \secondcell{38.1} & \secondcell{37.2} & \secondcell{36.9} & \secondcell{38.1} & \secondcell{$-2.2$} \\
\rowcolor{gray!30!white}
\textbf{AFD (ours)} & \bestcell{41.8} & \bestcell{39.6} & \bestcell{38.5} & \bestcell{38.1} & \bestcell{39.5} & \bestcell{$-1.6$} \\
\bottomrule
\end{tabular}
\end{minipage}
\hfill
\begin{minipage}{0.43\textwidth}
\centering
\textbf{(b) Non-continual references (higher is better)}
\vspace{0.4em}

\begin{tabular}{lcc}
\toprule
\multicolumn{3}{c}{\emph{Complex reasoning}} \\
\midrule
Method & CVQA EM & 11-VideoQA Acc. \\
\midrule
ColPro~\cite{ColPro24} & 55.8 & 60.2 \\
LTR~\cite{LTR25} & \thirdcell{58.7} & \thirdcell{63.1} \\
VQAGuider~\cite{VQAGuider25} & \secondcell{61.3} & \secondcell{66.5} \\
\rowcolor{gray!30!white}
\textbf{AFD (ours)} & \bestcell{62.8} & \bestcell{67.4} \\
\bottomrule
\end{tabular}

\vspace{0.9em}
\begin{tabular}{lcc}
\toprule
\multicolumn{3}{c}{\emph{Long-video stress test}} \\
\midrule
Method & VideoMME & MLVU \\
\midrule
VideoLLaMB~\cite{VideoLLaMB25} & \thirdcell{60.1} & \thirdcell{55.9} \\
LongVU~\cite{LongVU25} & \secondcell{61.2} & \secondcell{57.3} \\
\rowcolor{gray!30!white}
\textbf{AFD (ours)} & \bestcell{61.7} & \bestcell{57.9} \\
\bottomrule
\end{tabular}
\end{minipage}
} 
\end{table*}

Beyond the domain-incremental and ViLCo results in Tables~\ref{tab:vidqa-main}–\ref{tab:vilco-main}, we further report the time-incremental iVQA protocol and the two non-continual settings (complex reasoning and long-video stress tests) to verify the robustness.

\vspace{1mm}
\noindent\textbf{Time-Incremental iVQA.}
As shown in Table~\ref{tab:aux-evals}(a), AFD achieves the best average accuracy (39.5\%) with the lowest forgetting ($-1.6$), outperforming DAM by \(\mathbf{+1.4}\) points and Bisecle by \(\mathbf{+1.9}\). Per-slice gains are consistent (S1–S4), indicating that affordance-stabilized video tokens alleviate temporal distribution drift.

\vspace{1mm}
\noindent\textbf{Complex Reasoning.}
In Table~\ref{tab:aux-evals}(b, top), AFD attains the highest scores on \textsc{CVQA} and \textsc{11-VideoQA}, slightly surpassing VQAGuider and LTR . These results suggest that AFD’s scheduler composes affordance evidence effectively for multi-step reasoning without explicit tool calls.

\vspace{1mm}
\noindent\textbf{Long-Video Stress Tests.}
Table~\ref{tab:aux-evals}(b, bottom) shows that AFD is better than specialized long-video systems under the same backbone family, while our approach remains architecture-light (no dedicated memory bridges or heavy token pruning).

\subsection{Ablation and Analysis}
\label{sec:ablation}

\vspace{1mm}
\noindent\textbf{Ablation results}
We perform single-factor ablations under the same setup as Section~\ref{sec:exp-setup}. Variants are : 
\ding{182} w/o affordance tokens(direct frame tokens to LLM); 
\ding{183} w/o router (uniform adapter mixing); 
\ding{184} fixed LoRA rank $r{=}8$ (no rank growth); 
\ding{185} w/o question-only replay ($\lambda_{\text{rep}}{=}0$); 
\ding{186} w/o ASR weak-alignment term in $\mathcal{L}_{\text{aff}}$; 
\ding{187} w/o teacher-consistency (KL) in $\mathcal{L}_{\text{aff}}$; 
\ding{188} hard sparsity Top-$L{=}1$; 
\ding{189} smaller memory budgets ($B_Q{=}2$k,\ $B_A{=}256$). Descriptions of these variants can be found in the supplementary materials.
From Table~\ref{tab:ablation-main}, \ding{182} produces the largest drop and increases forgetting by 1.5, highlighting the central role of a stable affordance space. \ding{183} and \ding{184} further confirm that instance-wise routing and conflict-triggered capacity are both important. Other modules also have a positive impact on performance.

\newcommand{\abl}[2]{#1\,{\scriptsize\textcolor{blue}{(#2)}}}

\begin{table*}[t]
\centering
\small
\setlength{\tabcolsep}{6pt}
\caption{\textbf{Single-factor ablations.} 
}
\label{tab:ablation-main}
\resizebox{\textwidth}{!}{%
\begin{tabular}{lccccccccc}
\toprule
\multirow{2}{*}{Variant} &
\multicolumn{4}{c}{\textbf{Domain-Incremental VideoQA}} &
\phantom{abc} &
\multicolumn{3}{c}{\textbf{ViLCo-Bench (Ego4D)}} \\
\cmidrule{2-5} \cmidrule{7-9}
 & Avg. Acc.$\uparrow$ & MSRVTT$\uparrow$ & ANet$\uparrow$ & Forget$\downarrow$ & &
 MQ R@1@0.5$\uparrow$ & NLQ R@1@0.5$\uparrow$ & VQ stAP@0.25$\uparrow$ \\
\midrule
\textbf{Full AFD} & 51.6 & 43.7 & 38.0 & $-1.8$ && 29.6 & 20.7 & 18.4 \\
\midrule
\ding{182}
& \abl{48.7}{-2.9} & \abl{41.1}{-2.6} & \abl{35.4}{-2.6} & \abl{$-3.3$}{+1.5} && \abl{28.0}{-1.6} & \abl{19.1}{-1.6} & \abl{16.9}{-1.5} \\
\ding{183}
& \abl{49.8}{-1.8} & \abl{42.3}{-1.4} & \abl{36.9}{-1.1} & \abl{$-2.6$}{+0.8} && \abl{28.5}{-1.1} & \abl{19.7}{-1.0} & \abl{17.5}{-0.9} \\
\ding{184}
& \abl{50.5}{-1.1} & \abl{42.9}{-0.8} & \abl{37.3}{-0.7} & \abl{$-2.3$}{+0.5} && \abl{28.9}{-0.7} & \abl{20.0}{-0.7} & \abl{17.8}{-0.6} \\
\ding{185}
& \abl{50.2}{-1.4} & \abl{43.0}{-0.7} & \abl{36.8}{-1.2} & \abl{$-2.8$}{+1.0} && \abl{28.6}{-1.0} & \abl{19.8}{-0.9} & \abl{17.6}{-0.8} \\
\ding{186}
& \abl{50.8}{-0.8} & \abl{43.2}{-0.5} & \abl{37.4}{-0.6} & \abl{$-2.2$}{+0.4} && \abl{29.0}{-0.6} & \abl{19.9}{-0.8} & \abl{17.9}{-0.5} \\
\ding{187}
& \abl{50.6}{-1.0} & \abl{43.1}{-0.6} & \abl{37.2}{-0.8} & \abl{$-2.4$}{+0.6} && \abl{28.8}{-0.8} & \abl{19.8}{-0.9} & \abl{17.7}{-0.7} \\
\ding{188}
& \abl{50.7}{-0.9} & \abl{43.0}{-0.7} & \abl{37.3}{-0.7} & \abl{$-2.4$}{+0.6} && \abl{28.9}{-0.7} & \abl{20.1}{-0.6} & \abl{17.8}{-0.6} \\
\ding{189}
& \abl{51.0}{-0.6} & \abl{43.4}{-0.3} & \abl{37.6}{-0.4} & \abl{$-2.1$}{+0.3} && \abl{29.2}{-0.4} & \abl{20.3}{-0.4} & \abl{18.0}{-0.4} \\
\bottomrule
\end{tabular}%
}
\end{table*}

\vspace{1mm}
\noindent\textbf{Are affordance tokens stable?}
\label{sec:aff-stability}
We validate the structural division of labor (\emph{stable} affordance head vs.\ \emph{plastic} LLM scheduler) from a representation perspective. 
Across tasks $k{=}1\!\ldots\!6$ (iVQA$\rightarrow$MSVD$\rightarrow$MSRVTT$\rightarrow$LSMDC$\rightarrow$ANet$\rightarrow$ \\ TGIF), we (i) track prototype drift as cosine distance per affordance prototype between consecutive tasks ($k{-}1\!\rightarrow\!k$), and (ii) compute coverage of verb/action clusters captured by Top-$L$ mixtures ($L\!\in\!\{1,4,6,8,10,12\}$).
From Fig.~\ref{fig:aff-stability}, prototype drift remains small and narrowly distributed across all task transitions, while adjacent-task CKA remains high, evidencing a stable, reusable affordance space. Verb/action coverage rises monotonically with Top-$L$ and plateaus around $L{=}8$, supporting the design choice that soft, sparse mixtures encode co-occurring affordances without destabilizing the head.

\begin{figure}[htbp]
    \centering
    \begin{subfigure}[t]{0.5\textwidth}
        \includegraphics[width=0.9\textwidth]{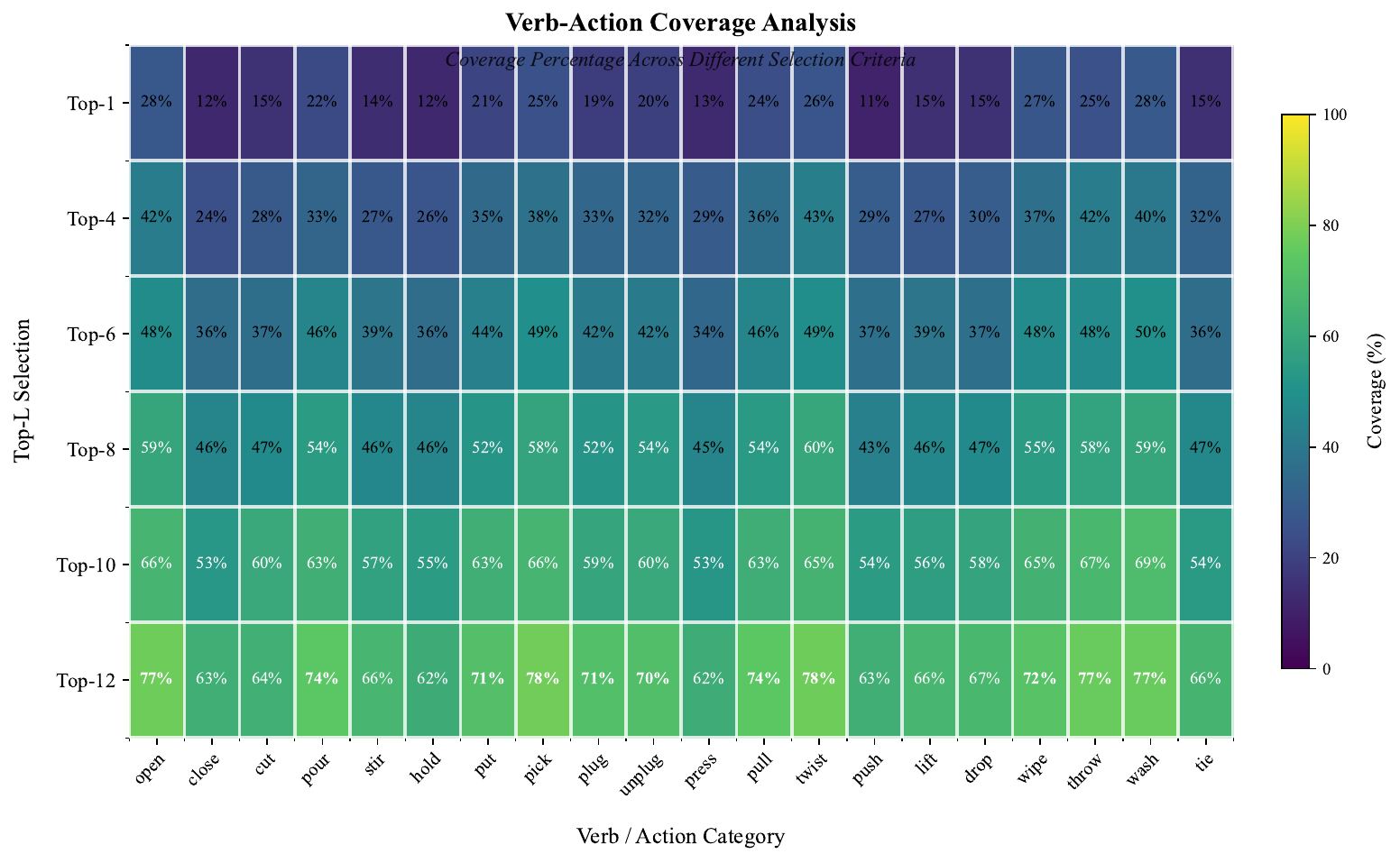}
        \caption{Verb-action coverage analysis at different Top-L thresholds}
        \label{fig:ridge}
    \end{subfigure}
    
    \begin{subfigure}[t]{0.5\textwidth}
        \includegraphics[width=0.9\textwidth]{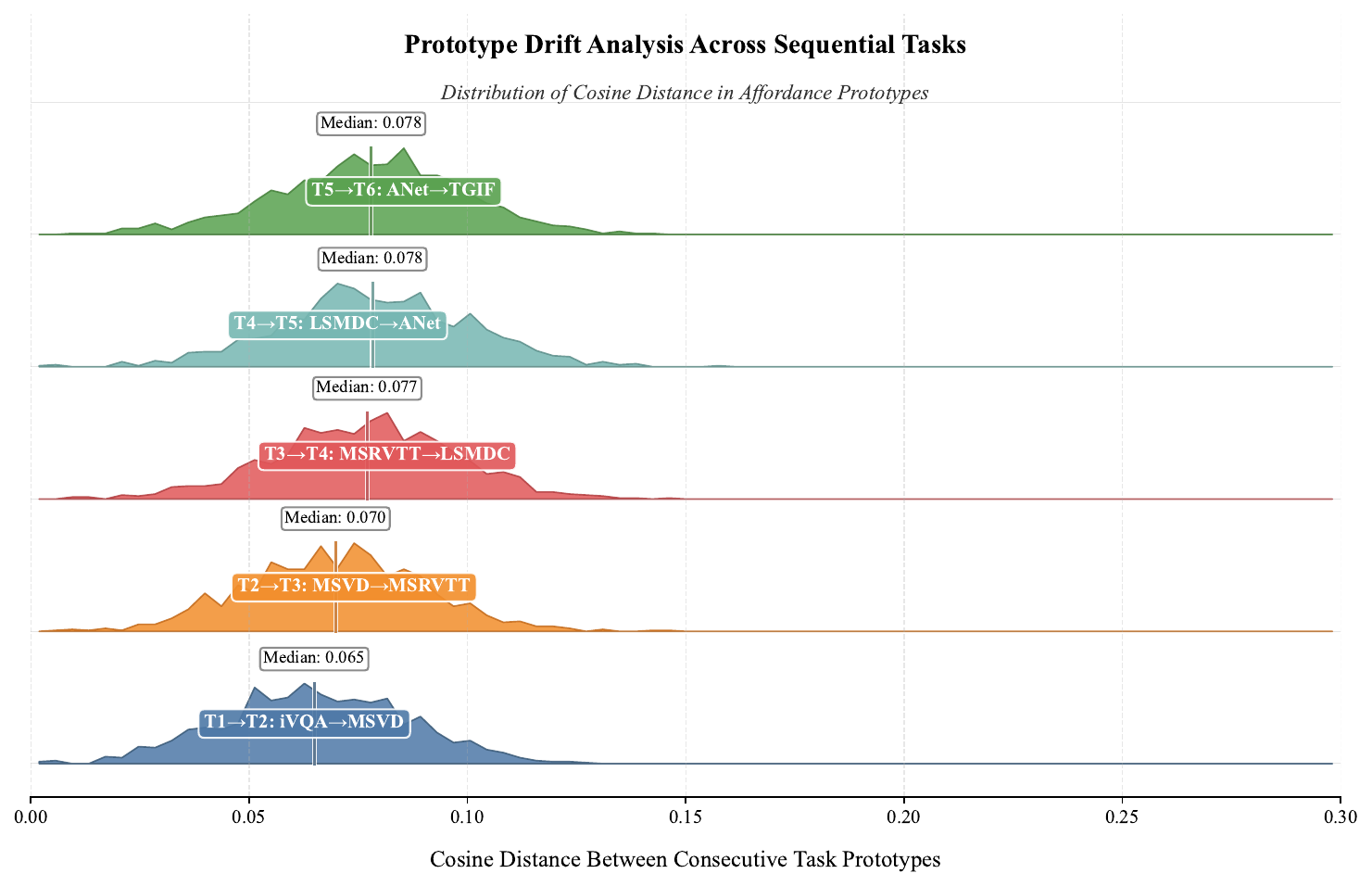}
        \caption{Prototype drift distributions across task transitions}
        \label{fig:heatmap}
    \end{subfigure}
    
    \caption{\textbf{Affordance stability and coverage.} (a) Soft Top-$L$ mixtures increase verb/action coverage without requiring more scheduler capacity.  (b) Drift distributions concentrate near zero with small spread across tasks, consistent with a slowly varying shared space.}
  \label{fig:aff-stability}
\end{figure}

\vspace{1mm}
\noindent\textbf{Case study} In Fig.~\ref{fig:case1}, AFD correctly anticipates the interaction, while Bisecle\cite{Bisecle25} focuses on incidental cues. This indicates that our affordance-first decomposition prioritizes object–action regularities over local appearance, improving causal anticipation.

\begin{figure}[htb]
	\centering
	\includegraphics[width=0.6\linewidth]{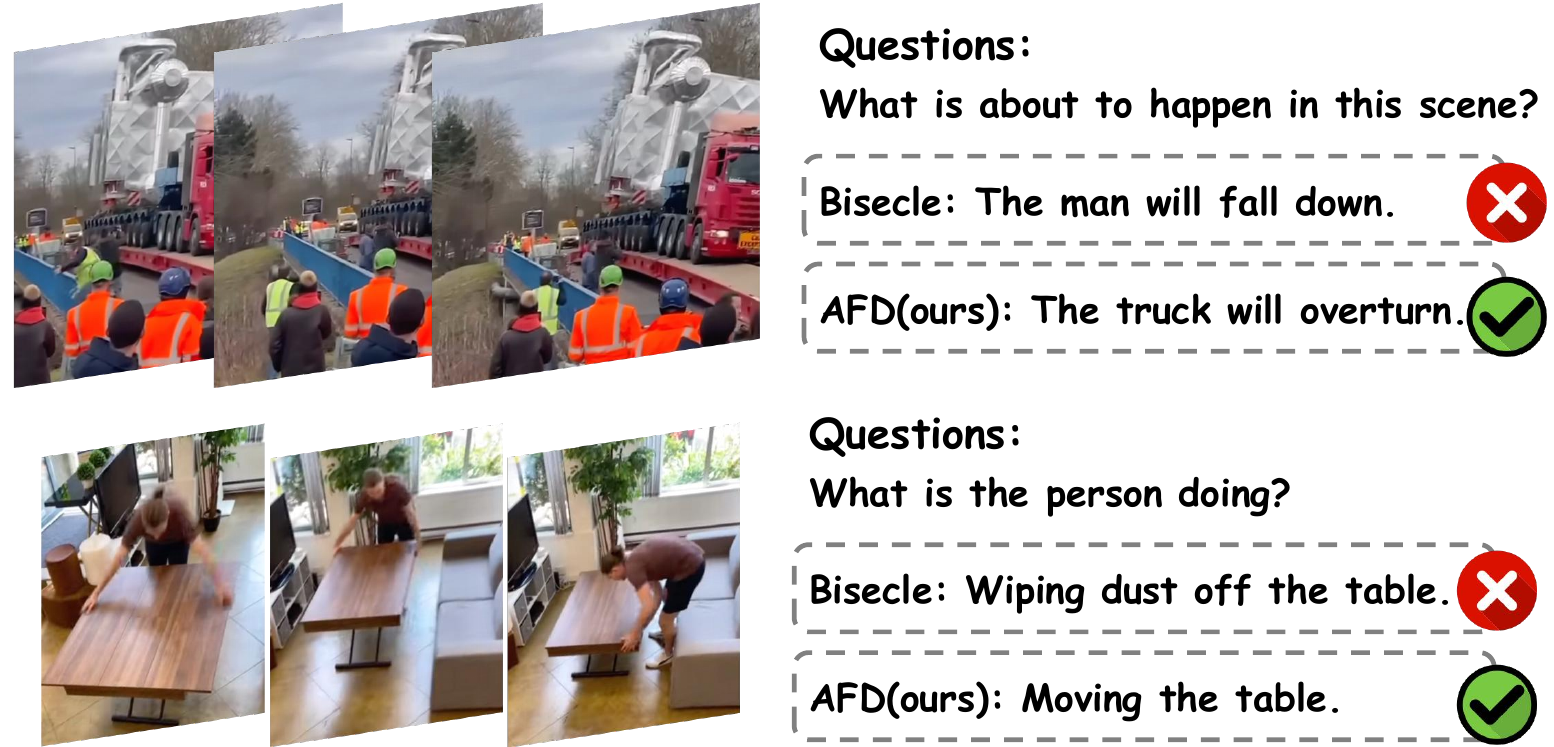}
	\caption{\textbf{Case study}.}
	\label{fig:case1}
\end{figure}

\section{Conclusion}
We addressed the central question of \emph{where stability should live and where plasticity should adapt} in continual video--language learning, proposing an \emph{affordance-first decomposition} with a slowly varying substrate and a query-routed, conflict-aware scheduler. 
Across ViLCo and domain/time-incremental VideoQA, this split yields state-of-the-art accuracy with substantially lower forgetting, and the diagnostics support the “slowly varying” substrate hypothesis. 

Future work will explore online affordance discovery, and multi-sensor extensions (e.g., audio).

\bibliographystyle{unsrtnat}
\bibliography{references}  

\clearpage

\appendix

\section{Supplementary Experimental Setup}
\label{sec:spES}

\subsection{Baselines}
\label{sec:baselines}
We benchmark AFD against widely adopted baselines reported in recent continual VideoQA works (e.g., DAM and Bisecle) and include standard continual learning references. To ensure comparability, all rehearsal-free methods are trained under the same task order, backbone family, and data splits as our method, when a baseline requires a specific backbone (e.g., LLM-adapter stacks), we keep its official configuration.

\paragraph{Naïve \& Upper-Bound References}
\begin{itemize}
  \item \textbf{Zero-Shot:} direct evaluation of the frozen pretrained model without any continual adaptation.
  \item \textbf{Sequential Fine-Tuning (Seq-FT):} train a single model task-by-task. This is a strong but forgetting-prone reference.
  \item \textbf{Multitask (Upper-Bound):} joint training on the union of all tasks (not a continual setting), instantiated with (i) \emph{Adapters} and (ii) \emph{Prompt Tuning}.
\end{itemize}

\paragraph{Regularization-Based CL (rehearsal-free)}
\begin{itemize}
  \item \textbf{EWC~\cite{EWC}:} quadratic penalty on parameter drift along Fisher-sensitive directions to preserve past tasks.
  \item \textbf{LwF~\cite{LwF}:} distillation from the previous model to the current one to mitigate forgetting.
\end{itemize}

\paragraph{Prompt-Based CL }
\begin{itemize}
  \item \textbf{L2P~\cite{L2P}:} a prompt pool with key–query retrieval for instance-wise prompt selection.
  \item \textbf{DualPrompt~\cite{Dualprompt}:} global \& task-specific prompt pairs to balance stability/plasticity.
  \item \textbf{CODA-Prompt (CoDA)~\cite{CODA-Prompt}:} context-dependent prompt adaptation for domain shifts.
  \item \textbf{S-Prompts~\cite{S-Prompts}:} Gaussian-mixture prompt selection for distribution-aware routing.
  \item \textbf{ProgPrompt~\cite{ProgPrompt}:} progressive accumulation of prompts across tasks.
\end{itemize}

\paragraph{Adapter/LoRA \& Model-Merging Families}
\begin{itemize}
  \item \textbf{DAM\cite{Dam25}:} dataset-wise adapters with a non-parametric router and \emph{dynamic adapter merging} at inference.
  \item \textbf{MoE baseline:} mixture-of-experts over adapter modules with learned gating (no merging).
  \item \textbf{Average Merging / RegMean:} weight-space model-merging baselines that combine per-task experts into a single model.
  \item \textit{Router Ablations (for methods with routing):} we also compare router choices commonly used in prior work—L2P’s key-memory retrieval, CODA-Prompt’s selector, S-Prompts’ GMM router, a learnable MLP router, and the non-parametric router used in DAM—to isolate routing effects from adaptation capacity.
\end{itemize}

\paragraph{LLM-Centric Continual VideoQA}
\begin{itemize}
  \item \textbf{ColPro\cite{ColPro24}:} collaborative prompting that injects question constraints, knowledge hints, and temporal awareness into an LLM for rehearsal-free continual VideoQA.
  \item \textbf{LAE~\cite{LAE23}:} Learning–Accumulation–Ensemble framework that reshapes general parameter-efficient tuning (e.g., prompts/adapters) for memory-free continual learning.
  \item \textbf{Bisecle\cite{Bisecle25}:} neuro-inspired \emph{binding \& separation} with multi-directional supervision and contrastive prompt learning atop an LLM–adapter backbone.
\end{itemize}

For fairness, we: (i) match the visual encoder and tokenizer family across methods in each table; (ii) keep rehearsal buffers disabled for rehearsal-free baselines; (iii) align memory/parameter budgets for prompts/adapters (same number or rank); (iv) report Multitask numbers separately as non-continual upper-bounds; and (v) when a baseline is tied to a specific backbone (e.g., LLaMA-Adapter stacks), we keep its official setting and do not mix results into backbones it does not support.

\subsection{Implementation Details}
\paragraph{Backbones and tokenization.}
For videos, we use a ViT-based spatiotemporal encoder at 16--32 input frames with $224^2$ resolution. Queries are embedded by the same LLM family used in AFD’s scheduler. We freeze the visual encoder and train LoRA adapters in the LLM scheduler for plasticity.

\paragraph{AFD hyperparameters.}
Affordance head uses a vocabulary $|\mathcal{V}_A|=1{,}024$ with temperature $\tau\!=\!0.07$, top-$L\!=\!8$ sparse renormalization, and embedding width $d_a\!=\!256$. Projections $W_K,W_V\in\mathbb{R}^{d_\text{model}\times d_a}$ map affordance tokens into the LLM space. Router operates at $\mathcal{S}\!=\!\{4,8,12,16,20,24\}$ with $m\!=\!4$ LoRA experts per layer; initial ranks $r_j^{(\ell)}\!=\!8$, max rank $r_{\max}\!=\!64$, conflict threshold $\tau_c\!=\!0.2$, gain $\gamma\!=\!6$ (Eq.~\ref{eq:rank-update}). Replay temperature $T_{\mathrm{kd}}\!=\!2.0$, confidence mask $\rho\!=\!0.6$ (Eq.~\ref{eq:replay}). Memory budgets: question memory $B_Q\!=\!8$k entries (de-duplicated by semantic hashing), affordance prototypes $B_A\!=\!1$k (per action cluster).

\paragraph{Optimization.}
We train with AdamW, base LR $2\!\times\!10^{-4}$ for adapters/projections and $1\!\times\!10^{-4}$ for the affordance head; cosine decay, warmup 2k steps. Weight decay $0.01$ except for LoRA/bias/scale. Batch size per GPU 24 clips (32 on MQ/VQ), sequence length 16--32 frames. Loss weights: $\lambda_{\text{aff}}\!=\!0.5$, $\lambda_{\text{rep}}\!=\!0.5$ (grid-searched on ViLCo val). Mixed precision (bfloat16) and gradient clipping at 1.0.

\paragraph{Hardware.}
All continual runs use 8$\times$A100-80GB. Long-video stress tests (VideoMME/MLVU) use 16$\times$A100-80GB with activation checkpointing.

\paragraph{Protocol details.}
For ViLCo-Bench, we strictly follow the official task orders and evaluation APIs and report Avg R@1/R@5 at IoU $0.3/0.5$ for MQ/NLQ, and tAP/stAP@0.25, recall, success for VQ, together with BwF~\cite{ViLCo24}. For DAM-style dataset incremental, we freeze the shared backbone and train dataset-specific adapters/banks as prescribed in each baseline. For time-incremental iVQA, we partition by upload date and report per-slice and averaged accuracy~\cite{Dam25}. For LTR and VQAGuider we reproduce their reasoning pipelines on the overlapping datasets and report their official metrics separately to avoid unfair mixing with continual scores.

\subsection{Ablation Variants}
We conduct single-factor ablations under the same setup as Section~\ref{sec:exp-setup}, averaging over 3 seeds. Only the named component is changed while all other settings (optimizer/backbone/task order/hyperparameters) are held fixed. The tested variants are:

\begin{itemize}
\item \textbf{\ding{182} — w/o Affordance tokens (direct frame tokens to LLM).} We bypass the affordance vocabulary and Top-$L$ mixture in Eqs.~\eqref{eq:aff-softmax}–\eqref{eq:aff-token} by setting $A_t \!:=\! f_{\text{st}}(X_t)$ (no $q_t$, no $E_A$). Keys/values are obtained by the same projections as Eq.~\eqref{eq:proj}: $K_t\!=\!W_K A_t$, $V_t\!=\!W_V A_t$. The scheduler and losses are unchanged.

\item \textbf{\ding{183} — w/o Router (uniform adapter mixing).} We disable instance-wise routing in Eq.~\eqref{eq:router} by replacing $\alpha^{(\ell)}\!=\!\mathrm{softmax}(W_r^{(\ell)}u)$ with a uniform mixture $\alpha^{(\ell)}\!=\!\tfrac{1}{m}\mathbf{1}$ at all routed layers $\ell\!\in\!\mathcal{S}$. LoRA experts remain present; rank growth (Eq.~\eqref{eq:rank-update}) stays enabled.

\item \textbf{\ding{184} — Fixed LoRA rank $r{=}8$ (no rank growth).} We freeze all adapter ranks at their initialization $r_j^{(\ell)}\!=\!8$ and disable the update rule in Eq.~\eqref{eq:rank-update}. The router in Eq.~\eqref{eq:router} remains active.

\item \textbf{\ding{185} — w/o Replay (question-only distillation disabled).} We set $\lambda_{\text{rep}}\!=\!0$ in Eq.~\eqref{eq:full}, ignore Eq.~\eqref{eq:replay}, and do not sample from the question memory $\mathcal{M}_Q$. Other losses and memories are unchanged.

\item \textbf{\ding{186} — w/o ASR weak alignment in $\mathcal{L}_{\text{aff}}$.} We remove the alignment term in Eq.~\eqref{eq:aff} by setting $\beta\!=\!0$ (teacher KL only). Gradients still update $\psi$ only.

\item \textbf{\ding{187} — w/o Teacher consistency (KL) in $\mathcal{L}_{\text{aff}}$.} We remove the KL term in Eq.~\eqref{eq:aff} by setting $\beta\!=\!1$ (ASR alignment only). Gradients still update $\psi$ only.

\item \textbf{\ding{188} — Hard sparsity Top-$L{=}1$.} In Eq.~\eqref{eq:aff-token}, we set $L\!=\!1$ and use $a^\star\!=\!\arg\max_a P_t(a)$ with $q_t(a^\star)\!=\!1$ (no soft mixture over $E_A$).

\item \textbf{\ding{189} — Smaller memories.} We reduce budgets from $(B_Q,B_A)\!=\!(8000,1000)$ to $(2000,256)$ while keeping the same de-duplication and sampling policies for $\mathcal{M}_Q$/$\mathcal{M}_A$.
\end{itemize}

\section{Theoretical Analysis}
\label{subsec:theory}

We give a stylized analysis of AFD in a continual-learning setting,
with the goal of making the separation between a stable
affordance head and a plastic routed scheduler mathematically explicit.  To keep notation compact in a double-column layout, we first introduce shorthands that will be used throughout.

\paragraph{Shorthands.}
Recall that tasks arrive sequentially $k=1,\dots,K$.  After finishing
task $k$ the model parameters are $(\psi^{(k)},\phi^{(k)})$.  We write
\begin{align}
  h^k &\;\triangleq\; h_{\psi^{(k)}},
  \qquad
  \phi^k \;\triangleq\; \phi^{(k)}, \\
  \Delta h^k &\;\triangleq\; h^{k} - h^{k-1},
  \qquad
  \Delta\phi^k \;\triangleq\; \phi^{k} - \phi^{k-1}.
\end{align}
For task $i$, the expected risk under $(\psi^{(k)},\phi^{(k)})$ is
denoted
\begin{equation}
  R_i^k \;\triangleq\; R_i(\psi^{(k)},\phi^{(k)}).
\end{equation}
Gradients of $R_i$ w.r.t.\ the scheduler at step $k$ are abbreviated as
\begin{equation}
  g_i^k \;\triangleq\;
  \nabla_{\phi} R_i(\psi^{(k)},\phi^{(k)}).
\end{equation}
The scheduler update direction at task $k$ is
\begin{equation}
  d^k \;\triangleq\;
  P_k \nabla_{\phi} L_k(\psi^{(k-1)},\phi^{(k-1)}),
\end{equation}
where $P_k$ is the orthogonal projector onto the routed LoRA subspace
for task $k$ and $L_k$ is the task$+$replay loss.  The effective update
is
\begin{equation}
  \Delta\phi^k
  \;=\;
  -\eta\,d^k,
  \qquad
  \eta>0.
  \label{eq:delta-phi-def}
\end{equation}

\paragraph{Forgetting measure and path length.}
For a past task $i\le K$, we define the forgetting at time $K$ as
\begin{equation}
  F_i^K
  \;\triangleq\;
  R_i^K - R_i^i.
  \label{eq:forgetting-def}
\end{equation}
We also introduce the cumulative affordance drift and scheduler
path length after task $i$:
\begin{align}
  B_h^{i:K}
  &\;\triangleq\;
  \sum_{k=i+1}^{K}
  \big\|\Delta h^k\big\|_{\mathrm{op}},
  \label{eq:Bh-def-compact}
  \\
  S_{i:K}
  &\;\triangleq\;
  \sum_{k=i+1}^{K}
  \big\|\Delta\phi^k\big\|_2^2.
  \label{eq:S-def-compact}
\end{align}
Here $\|\cdot\|_{\mathrm{op}}$ is the operator norm induced by the
Euclidean norm.

\paragraph{Assumptions.}
We adopt standard regularity conditions.

\begin{description}
  \item[(A1) Lipschitz loss and bounded predictions.]
  There exist $L_{\ell},B_f>0$ such that for any $y$ and predictions
  $\hat{y}_1,\hat{y}_2$,
  \begin{align}
    \big|
      \ell(\hat{y}_1,y) - \ell(\hat{y}_2,y)
    \big|
    &\le
    L_{\ell}\,|\hat{y}_1-\hat{y}_2|,
    \\
    |f_{\Theta}(V,q)|
    &\le
    B_f,
    \quad
    \forall (V,q).
  \end{align}

  \item[(A2) Lipschitz in the affordance operator.]
  There exists $L_h>0$ such that for any task $i$ and any
  $(\psi,\psi',\phi)$,
  \begin{equation}
    |R_i(\psi',\phi) - R_i(\psi,\phi)|
    \;\le\;
    L_h\,
    \big\|
      h_{\psi'} - h_{\psi}
    \big\|_{\mathrm{op}}.
    \label{eq:Lh-lip-compact}
  \end{equation}

  \item[(A3) Smoothness and bounded gradients in $\phi$.]
  For each $i$ and $\psi$, the map $\phi\mapsto R_i(\psi,\phi)$ is
  $L_{\phi}$-smooth:
  \begin{align}
    R_i(\psi,\phi')
    &\le
    R_i(\psi,\phi)
    +
    \langle g_i(\psi,\phi),\phi'-\phi\rangle
    \nonumber\\
    &\quad
    + \tfrac{L_{\phi}}{2}
      \|\phi'-\phi\|_2^2,
    \label{eq:Lphi-smooth-compact}
  \end{align}
  and gradients are uniformly bounded:
  \begin{equation}
    \|g_i(\psi,\phi)\|_2 \le G_{\phi},
    \quad
    \forall i,(\psi,\phi).
    \label{eq:Gphi-bound}
  \end{equation}

  \item[(A4) Conflict-aware routing.]
  The question-only replay and conflict-aware router in AFD are designed
  so that the update direction $d^k$ is not strongly anti-aligned with
  any past task gradient.  We encode this via a cosine bound: there
  exists $\rho\in[0,1)$ such that for all $k>i$,
  \begin{equation}
    \langle g_i^{k-1}, d^k\rangle
    \;\ge\;
    -\rho\,
    \|g_i^{k-1}\|_2\,
    \|d^k\|_2.
    \label{eq:rho-conflict-compact}
  \end{equation}
\end{description}

\subsection{Single-Step Effect on a Past Task}

We first bound the change in $R_i$ incurred by a single update at step
$k>i$.  To keep notation uncluttered, set
\begin{equation}
  \begin{split}
    h^{k-1} &= h^{\mathrm{old}},
    \quad
    h^k = h^{\mathrm{new}},
    \\
    \phi^{k-1} &= \phi^{\mathrm{old}},
    \quad
    \phi^k = \phi^{\mathrm{new}}.
  \end{split}
\end{equation}

\begin{lemma}[Single-step bound]
\label{lem:single-step-compact}
Under (A1)--(A4), for any $k>i$,
\begin{align}
  R_i^k - R_i^{k-1}
  &\le
  L_h\,
  \big\|
    \Delta h^k
  \big\|_{\mathrm{op}}
  \nonumber\\
  &\quad
  + \rho\,G_{\phi}\,
    \|\Delta\phi^k\|_2
  + \frac{L_{\phi}}{2\eta}\,
    \|\Delta\phi^k\|_2^2.
  \label{eq:single-step-compact}
\end{align}
\end{lemma}

\begin{proof}
We split the increment into a representation part and a scheduler part:
\begin{align}
  R_i^k - R_i^{k-1}
  &=
  \big[
    R_i(h^{\mathrm{new}},\phi^{\mathrm{new}})
    - R_i(h^{\mathrm{old}},\phi^{\mathrm{new}})
  \big]
  \nonumber\\
  &\quad
  +
  \big[
    R_i(h^{\mathrm{old}},\phi^{\mathrm{new}})
    - R_i(h^{\mathrm{old}},\phi^{\mathrm{old}})
  \big].
  \label{eq:delta-split-compact}
\end{align}

\emph{Representation term.}
By \eqref{eq:Lh-lip-compact},
\begin{equation}
  R_i(h^{\mathrm{new}},\phi^{\mathrm{new}})
  - R_i(h^{\mathrm{old}},\phi^{\mathrm{new}})
  \le
  L_h\,
  \|\Delta h^k\|_{\mathrm{op}}.
  \label{eq:rep-term-compact}
\end{equation}

\emph{Scheduler term.}
Apply \eqref{eq:Lphi-smooth-compact} with
$(\psi,\phi,\phi')=(\psi^{(k-1)},\phi^{\mathrm{old}},\phi^{\mathrm{new}})$:
\begin{align}
  R_i(h^{\mathrm{old}},\phi^{\mathrm{new}})
  &\le
  R_i(h^{\mathrm{old}},\phi^{\mathrm{old}})
  + \langle g_i^{k-1},\Delta\phi^k\rangle
  \nonumber\\
  &\quad
  + \tfrac{L_{\phi}}{2}
    \|\Delta\phi^k\|_2^2.
  \label{eq:sched-smooth-compact}
\end{align}
Rearranging,
\begin{align}
  &
  R_i(h^{\mathrm{old}},\phi^{\mathrm{new}})
  - R_i(h^{\mathrm{old}},\phi^{\mathrm{old}})
  \nonumber\\
  &\quad\le
  \langle g_i^{k-1},\Delta\phi^k\rangle
  + \tfrac{L_{\phi}}{2}
    \|\Delta\phi^k\|_2^2.
  \label{eq:sched-term-pre-conflict}
\end{align}
Using \eqref{eq:delta-phi-def} and
$\Delta\phi^k=-\eta d^k$, equation
\eqref{eq:rho-conflict-compact} gives
\begin{align}
  \langle g_i^{k-1},\Delta\phi^k\rangle
  &= -\eta\,\langle g_i^{k-1},d^k\rangle
  \nonumber\\
  &\le
  \eta\,
  \rho\,\|g_i^{k-1}\|_2\,\|d^k\|_2
  \nonumber\\
  &\le
  \tfrac{\rho\,G_{\phi}}{\eta}\,
  \|\Delta\phi^k\|_2,
  \label{eq:conflict-to-delta}
\end{align}
where we used $\|d^k\|_2=\|\Delta\phi^k\|_2/\eta$ in the last step.
Substituting \eqref{eq:conflict-to-delta} into
\eqref{eq:sched-term-pre-conflict} yields
\begin{align}
  &
  R_i(h^{\mathrm{old}},\phi^{\mathrm{new}})
  - R_i(h^{\mathrm{old}},\phi^{\mathrm{old}})
  \nonumber\\
  &\quad\le
  \tfrac{\rho\,G_{\phi}}{\eta}\,
  \|\Delta\phi^k\|_2
  + \tfrac{L_{\phi}}{2}
    \|\Delta\phi^k\|_2^2.
  \label{eq:sched-term-compact}
\end{align}
Combining \eqref{eq:rep-term-compact} and
\eqref{eq:sched-term-compact} with
\eqref{eq:delta-split-compact} gives
\eqref{eq:single-step-compact}. \hfill $\square$
\end{proof}

\subsection{Task-wise Forgetting Bound}

We now sum Lemma~\ref{lem:single-step-compact} over all updates after
task $i$ and control the linear path-length term via
Cauchy--Schwarz.

\begin{theorem}[Task-wise forgetting]
\label{thm:forgetting-compact}
Under \textnormal{(A1)--(A4)}, for any $1\le i\le K$,
\begin{align}
  F_i^K
  &\le
  L_h\,
  B_h^{i:K}
  \nonumber\\
  &\quad
  + \rho\,G_{\phi}\,
    \sqrt{K-i}\,
    \sqrt{S_{i:K}}
  + \frac{L_{\phi}}{2\eta}\,
    S_{i:K}.
  \label{eq:forgetting-bound-compact}
\end{align}
\end{theorem}

\begin{proof}
By telescoping,
\begin{align}
  F_i^K
  &= R_i^K - R_i^i
  \nonumber\\
  &= \sum_{k=i+1}^{K}
     \big(R_i^k - R_i^{k-1}\big).
  \label{eq:forgetting-telescope-compact}
\end{align}
Applying Lemma~\ref{lem:single-step-compact} term-wise,
\begin{align}
  F_i^K
  &\le
  \sum_{k=i+1}^{K}
  L_h\,\|\Delta h^k\|_{\mathrm{op}}
  \nonumber\\
  &\quad
  + \rho\,G_{\phi}
    \sum_{k=i+1}^{K}
    \|\Delta\phi^k\|_2
  + \frac{L_{\phi}}{2\eta}
    \sum_{k=i+1}^{K}
    \|\Delta\phi^k\|_2^2.
\end{align}
The first sum is exactly $L_h B_h^{i:K}$.  For the second, Cauchy--Schwarz
yields
\begin{align}
  \sum_{k=i+1}^{K}\|\Delta\phi^k\|_2
  &\le
  \sqrt{K-i}\,
  \Big(
    \sum_{k=i+1}^{K}\|\Delta\phi^k\|_2^2
  \Big)^{1/2}
  \nonumber\\
  &= \sqrt{K-i}\,\sqrt{S_{i:K}}.
  \label{eq:path-cs-compact}
\end{align}
The last sum is $S_{i:K}$ by definition.  Substituting these into
\eqref{eq:forgetting-telescope-compact} gives
\eqref{eq:forgetting-bound-compact}.\hfill $\square$
\end{proof}

The bound \eqref{eq:forgetting-bound-compact} separates three
contributions:
\begin{itemize}
  \item $L_h B_h^{i:K}$: forgetting due to \emph{affordance drift},
        controlled by the stability loss $\mathcal{L}_{\text{aff}}$.
  \item $\rho\,G_{\phi}\sqrt{K-i}\sqrt{S_{i:K}}$:
        a \emph{first-order interference term} coupling the worst-case
        negative cosine $\rho$ with the scheduler path length
        $S_{i:K}$.
  \item $(L_{\phi}/(2\eta)) S_{i:K}$:
        a \emph{second-order curvature term} that is small when the
        loss is smooth and updates are moderate.
\end{itemize}
In the ideal regime where (i) the affordance head is nearly frozen
($B_h^{i:K}\approx 0$) and (ii) conflict-aware routing makes gradient
subspaces almost orthogonal ($\rho\approx 0$), the dominant term is the
curvature term, which vanishes as the effective path length
$S_{i:K}$ shrinks.

\subsection{Regret of the Routed Scheduler}

We briefly relate the scheduler dynamics to online regret.  Consider the
sequence of convex losses
\begin{equation}
  \ell_k(\phi)
  \;\triangleq\;
  R_k(\psi^{(k-1)},\phi),
\end{equation}
and the projected OGD update
\begin{equation}
  \phi^k
  \;=\;
  \Pi\big(\phi^{k-1} - \eta d^k\big),
\end{equation}
where $\Pi$ is projection onto a convex set of diameter at most $D$.
For a comparator $\phi^{\star}$, the static regret is
\begin{equation}
  \mathrm{Reg}_{K}
  \;\triangleq\;
  \sum_{k=1}^{K}\ell_k(\phi^{k-1})
  -
  \sum_{k=1}^{K}\ell_k(\phi^{\star}).
\end{equation}

\begin{proposition}[Scheduler regret]
\label{prop:regret-compact}
Assume each $\ell_k$ is convex, gradients are bounded by $G_{\phi}$, and
the feasible set has diameter $D$.  If
$\eta = D/(G_{\phi}\sqrt{K})$, then
\begin{equation}
  \mathrm{Reg}_{K}
  \;\le\;
  D G_{\phi}\sqrt{K},
\end{equation}
and the squared path length satisfies
$S_{0:K} \le D^2$.
\end{proposition}

\begin{proof}
This is the standard analysis of projected OGD.  The key inequality is
\begin{align}
  &
  \|\phi^k - \phi^{\star}\|_2^2
  \nonumber\\
  &\le
  \|\phi^{k-1} - \phi^{\star}\|_2^2
  - 2\eta\,
    \langle\nabla\ell_k(\phi^{k-1}),\phi^{k-1}-\phi^{\star}\rangle
  \nonumber\\
  &\quad
  + \eta^2 G_{\phi}^2,
\end{align}
which, after rearranging, using convexity, and summing over $k$, yields
\begin{equation}
  \mathrm{Reg}_{K}
  \le
  \frac{\|\phi^0-\phi^{\star}\|_2^2}{2\eta}
  +
  \frac{\eta}{2} G_{\phi}^2 K
  \;\le\;
  \frac{D^2}{2\eta}
  + \frac{\eta}{2} G_{\phi}^2 K.
\end{equation}
Choosing $\eta = D/(G_{\phi}\sqrt{K})$ minimizes the RHS and gives the
claimed regret bound; with this choice the total movement is at most
$D$, hence $S_{0:K}\le D^2$.\hfill $\square$
\end{proof}

Combining Theorem~\ref{thm:forgetting-compact} and
Proposition~\ref{prop:regret-compact}, we see that the same mechanisms
that control the scheduler's online regret---bounded gradients and
short path length in routed low-rank subspaces---also control
catastrophic forgetting once the affordance substrate is stabilized.
This provides a theoretical underpinning for the design of AFD: by
confining plasticity to conflict-aware low-rank updates and making the
affordance head slowly varying, AFD simultaneously enjoys low regret on
a nonstationary stream and tight forgetting guarantees on past tasks.

\section{Additional results}

\subsection{Hyperparameter Sensitivity}
\label{sec:sensitivity}
We study the five most important hyperparameters while keeping all others fixed. As shown in Figure~\ref{fig:sensitivity}: 
\textbf{(i) Broad plateaus.} Across wide ranges, performance remains essentially flat. DI-Avg varies by only \(\approx 0.2\text{--}0.4\%\) and \textit{Forget} stays around \(-1.8\) to \(-2.0\). \textbf{(ii) Edge effects.} At the boundaries, very sparse mixtures (\(L{=}1\)) or overly permissive temperatures (\(\tau{=}0.10\)), shallow routing (\(|\mathcal{S}|{=}3\)), tight capacity (\(r_{\max}{=}32\)), or aggressive growth triggers (\(\tau_c{=}0.10\)) cause small but consistent drops and slightly higher forgetting.

\begin{figure*}[ht]
	\centering
	\includegraphics[width=\linewidth]{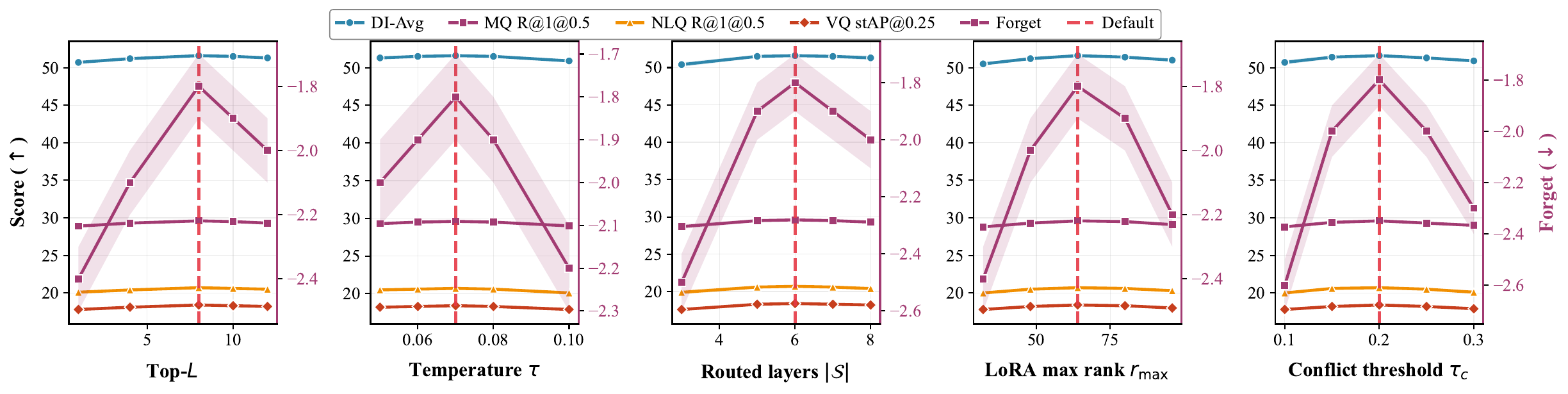}
	\caption{Sensitivity on the five key hyperparameters. Defaults marked with \(^{\star}\). AFD remains stable across broad ranges; extremes show mild degradation.}
	\label{fig:sensitivity}
\end{figure*}

\subsection{Drift--Forgetting Diagnostics for the Affordance Substrate}
\label{sec:drift-forgetting}

To empirically validate the theoretical decomposition in
Sec.~\ref{subsec:theory}, we study how task-wise forgetting correlates
with the drift of the shared affordance head across tasks.  We focus on
the 6-step domain-incremental VideoQA protocol
and the 4-slice time-incremental iVQA protocol (S1$\rightarrow$S4),
yielding $10$ tasks in total.  For each task $i$ and method, we compute:

\begin{itemize}
\item an empirical proxy for the cumulative affordance drift
      $\widehat B_{h,i}$; and
\item the magnitude of forgetting $\widehat F_i$ on that task.
\end{itemize}

We report results for Full AFD and three key ablations:
\ding{182} w/o affordance tokens,
\ding{186} w/o ASR alignment in $\mathcal{L}_{\text{aff}}$, and
\ding{185} w/o question-only replay.
For each task $i$, we approximate the cumulative affordance drift as a
prototype-level cosine distance across successive tasks.  Let
$\mathcal{P}$ denote the set of affordance prototypes, and let
$\mu_p^{(k)}\in\mathbb{R}^{d_a}$ be the learned embedding of prototype
$p\in\mathcal{P}$ after training on task $k$.  We define
\begin{equation}
  \widehat B_{h,i}
  \;\triangleq\;
  \sum_{k=i+1}^{K}
  \frac{1}{|\mathcal{P}|}
  \sum_{p\in\mathcal{P}}
  \Big(
    1 - \cos\big(\mu_p^{(k)},\mu_p^{(k-1)}\big)
  \Big),
  \label{eq:Bh-practical}
\end{equation}
which serves as a task-wise surrogate for the operator-norm drift
$B_h^{i:K}$ in Eq.~\eqref{eq:Bh-def-compact}.  Larger
$\widehat B_{h,i}$ indicates a less stable affordance substrate for
task $i$.

Following the standard CL literature, we define forgetting on task $i$
as the drop between its best and final performance.  Let
$A_i^{\max}$ be the maximum validation accuracy observed on task $i$
over the training trajectory, and let $A_i^{\text{final}}$ be the
accuracy after finishing all $K$ tasks.  We then use the magnitude
\begin{equation}
  \widehat F_i
  \;\triangleq\;
  A_i^{\max} - A_i^{\text{final}},
  \label{eq:forget-mag-practical}
\end{equation}
which is non-negative and directly comparable across methods
(smaller is better).

For each of the $10$ tasks (6 domain-incremental datasets and 4
time-slices), we collect $(\widehat B_{h,i},\widehat F_i)$ pairs for:

\begin{itemize}
  \item \textbf{AFD (full)} --- our full model;
  \item \textbf{\ding{182} w/o affordance tokens} --- direct frame tokens
        to the LLM scheduler;
  \item \textbf{\ding{186} w/o ASR alignment} --- $\beta=0$ in
        $\mathcal{L}_{\text{aff}}$ (teacher consistency only);
  \item \textbf{\ding{185} w/o replay} --- $\lambda_{\text{rep}}=0$
        (no question-only replay).
\end{itemize}

This yields $40$ points in the $(\widehat B_{h,i},\widehat F_i)$ plane.
We additionally approximate the scheduler path-length proxy
$\widehat S_i$ by accumulating the squared Frobenius norm of LoRA
updates on task $i$, in line with $S_{i:K}$ in
Eq.~\eqref{eq:S-def-compact}, and use it in a 3D diagnostic plot.

\begin{figure}[ht]
  \centering
  \includegraphics[width=0.6\linewidth]{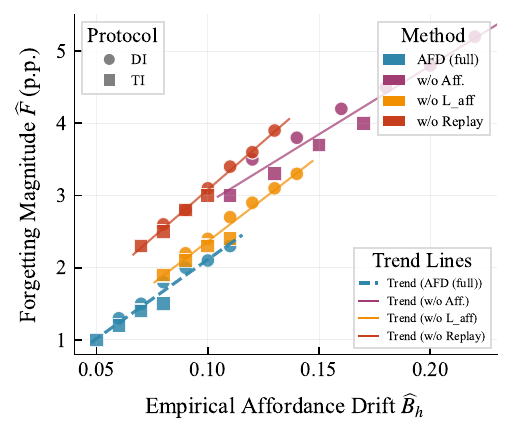}
  \caption{\textbf{Affordance drift vs.\ forgetting.}
  Each point is one task $i$ from the 6-step domain-incremental
  VideoQA or 4-slice time-incremental iVQA protocol, colored by method.}
  \label{fig:drift-forget-2d}
\end{figure}

\begin{figure}[ht]
  \centering
  \includegraphics[width=0.6\linewidth]{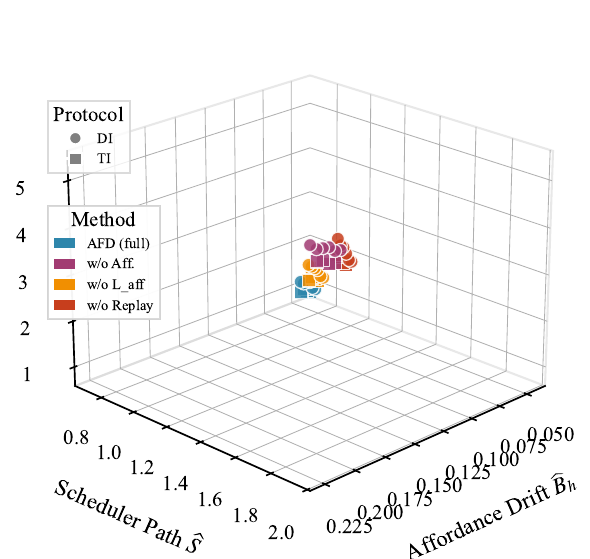}
  \caption{\textbf{Joint geometry of drift, scheduler path-length, and
  forgetting.}  3D scatter of $(\widehat B_{h,i},\widehat S_i,
  \widehat F_i)$ across tasks and methods.  Full AFD occupies the
  lower-left region (small drift, small path-length, small forgetting),
  while ablations move along both the drift and path-length axes toward
  higher forgetting, consistent with the theoretical bound in
  Eq.~\eqref{eq:forgetting-bound-compact}.}
  \label{fig:drift-path-forget-3d}
\end{figure}

Fig.~\ref{fig:drift-forget-2d} shows the 2D scatter of
$(\widehat B_{h,i},\widehat F_i)$ across tasks and methods, along with
least-squares fitted lines per method.  Points are colored by method and
annotated by task.  Fig.~\ref{fig:drift-path-forget-3d} extends this to
a 3D view, where each point encodes
$(\widehat B_{h,i},\widehat S_i,\widehat F_i)$ for a given task and
method.
We highlight three observations:
\begin{itemize}
  \item \emph{Drift--forgetting alignment.}
        Across all methods, tasks with larger
        $\widehat B_{h,i}$ exhibit systematically larger
        $\widehat F_i$, with a markedly tighter correlation for AFD and
        its variants.  Full AFD achieves the smallest drift and the
        lowest forgetting on most tasks, supporting the
        ``stable-substrate'' design.
  \item \emph{Effect of affordance ablations.}
        Removing affordance tokens (\ding{182}) shifts points toward
        the high-drift/high-forgetting corner, and both removing ASR
        alignment (\ding{186}) and replay (\ding{185}) produce
        intermediate degradations.  This mirrors the theoretical
        decomposition: weakening $\mathcal{L}_{\text{aff}}$ increases
        $B_h$, while removing replay affects the scheduler path-length
        and its interference.
  \item \emph{Joint geometry with scheduler path-length.}
        In Fig.~\ref{fig:drift-path-forget-3d}, methods form a
        structured manifold: full AFD lies near the origin in both
        $\widehat B_{h,i}$ and $\widehat S_i$, whereas ablations move
        along both directions toward larger forgetting, qualitatively
        matching the bound in Eq.~\eqref{eq:forgetting-bound-compact}.
\end{itemize}

\subsection{Scheduler Path-Length vs Forgetting and Backward Transfer}
\label{sec:path-length-forgetting}

The second term in Theorem~\ref{thm:forgetting-compact} shows that the
scheduler path-length $S_{i:K}$---i.e., the cumulative norm of
low-rank updates---directly contributes to forgetting.  We now
empirically probe this effect by correlating a practical proxy
$\widehat S^{(m)}$ with method-level forgetting and backward transfer
(BWT) in the domain-incremental VideoQA setting.

For a given method $m$ trained on the 6-step domain-incremental
VideoQA sequence, we denote by
$\{\phi^{k,(m)}\}_{k=0}^{K}$ the scheduler parameters after each
task, and by
$\Delta\phi^{k,(m)} = \phi^{k,(m)} - \phi^{k-1,(m)}$ the effective
update at step $k$.  We define the method-level path-length proxy as
\begin{equation}
  \widehat S^{(m)}
  \;\triangleq\;
  \sum_{k=1}^{K}
  \big\|
    \Delta\phi^{k,(m)}
  \big\|_F^2,
  \label{eq:Sm-proxy}
\end{equation}
which is an empirical counterpart of $S_{i:K}$ in
Eq.~\eqref{eq:S-def-compact}.  Larger $\widehat S^{(m)}$ indicates a
longer path in parameter space.

For forgetting, we aggregate the standard CL metric over tasks.  Let
$A_i^{\max,(m)}$ be the best validation accuracy achieved on task $i$
for method $m$, and $A_i^{\text{final},(m)}$ the accuracy on task $i$
after training all $K$ tasks.  We define the method-level forgetting
magnitude
\begin{equation}
  |\widehat F^{(m)}|
  \;\triangleq\;
  \frac{1}{K}
  \sum_{i=1}^{K}
  \big(
    A_i^{\max,(m)} - A_i^{\text{final},(m)}
  \big),
  \label{eq:Fm-proxy}
\end{equation}
in percentage points (p.p.).  Similarly, we approximate the backward
transfer as
\begin{equation}
  \widehat{\mathrm{BWT}}^{(m)}
  \;\triangleq\;
  \frac{1}{K-1}
  \sum_{i=1}^{K-1}
  \big(
    A_i^{\text{final},(m)} - A_i^{\max,(m)}
  \big),
\end{equation}
and use its magnitude
$|\widehat{\mathrm{BWT}}^{(m)}|$ (smaller is better) for plotting.

We instantiate this diagnostic on the domain-incremental VideoQA
experiment (6 datasets) for the following methods, each averaged over
three random seeds:

\begin{itemize}
  \item \textbf{AFD (full)} --- full model;
  \item \textbf{w/o Router} --- uniform adapter mixing
        (no instance-wise routing, cf.\ \ding{183});
  \item \textbf{Fixed Rank} --- LoRA ranks fixed at $r{=}8$
        (no rank growth, cf.\ \ding{184});
  \item \textbf{w/o Replay} --- $\lambda_{\text{rep}}{=}0$
        (no question-only replay, cf.\ \ding{185}).
\end{itemize}

For each (method, seed) pair we record the proxy path-length
$\widehat S^{(m)}$, the average forgetting magnitude
$|\widehat F^{(m)}|$, the average backward transfer magnitude
$|\widehat{\mathrm{BWT}}^{(m)}|$, and the average accuracy.

\begin{figure}[ht]
  \centering
  \includegraphics[width=0.6\linewidth]{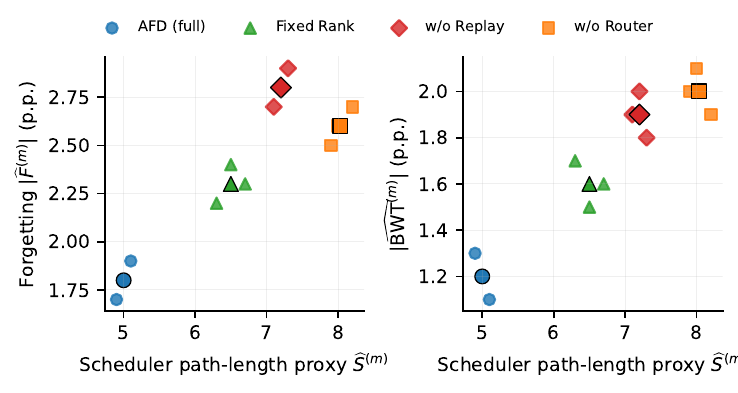}
  \caption{\textbf{Scheduler path-length vs.\ forgetting and BWT.}
  Each point corresponds to one seed of a given method on the
  domain-incremental VideoQA protocol. }
  \label{fig:path-forgetting}
\end{figure}

Fig.~\ref{fig:path-forgetting} summarizes the relationship between
scheduler path-length and stability.  The left panel plots forgetting
versus $\widehat S^{(m)}$, while the right panel plots
$|\widehat{\mathrm{BWT}}^{(m)}|$ versus $\widehat S^{(m)}$.
Three trends emerge:
\begin{itemize}
  \item \emph{Monotone S--F trade-off.}
        Across methods, larger $\widehat S^{(m)}$ correlates with
        larger forgetting magnitudes $|\widehat F^{(m)}|$.  Full AFD
        occupies the lower-left corner (mean
        $\widehat S^{(\text{AFD})}\approx 5.0$,
        $|\widehat F^{(\text{AFD})}|\approx 1.8$ p.p.),
        while w/o Router and w/o Replay lie around
        $\widehat S^{(m)}\approx 7.8\text{--}8.1$ with
        $|\widehat F^{(m)}|\approx 2.6\text{--}2.8$ p.p.,
        consistent with the path-length term in
        Eq.~\eqref{eq:forgetting-bound-compact}.
  \item \emph{Router and rank growth reduce path-length.}
        Disabling the router or rank growth forces the scheduler to
        traverse longer paths in parameter space to fit the same task
        stream.  The Fixed-Rank variant exhibits intermediate
        $\widehat S^{(m)}$ and forgetting, suggesting that conflict-aware
        rank expansion indeed acts as a ``shortest-path'' mechanism in
        the routed subspace.
  \item \emph{BWT follows the same geometry.}
        The right panel shows that $|\widehat{\mathrm{BWT}}^{(m)}|$
        is smallest for full AFD and grows with $\widehat S^{(m)}$
        for the ablations, indicating that shorter, conflict-aware
        trajectories not only reduce forgetting but also keep backward
        interference small.
\end{itemize}

Together with the drift diagnostics in Sec.~\ref{sec:drift-forgetting},
this experiment empirically supports the theoretical view that AFD
controls catastrophic forgetting through both a stable affordance
subspace and a geometrically efficient scheduler trajectory.

\subsection{Cross-Task CKA of Affordance vs.\ Visual Features}
\label{sec:cka-affordance}

To directly probe the ``stable substrate'' hypothesis behind AFD, we
compare how fast the raw visual tokens and the affordance tokens drift across tasks using centered kernel alignment (CKA).  The goal is to show that, under the same continual protocol, the
affordance space evolves significantly more slowly than the backbone
feature space.

We use the 6-step domain-incremental VideoQA protocol
(iVQA$\rightarrow$MSVD$\rightarrow$MSRVTT$\rightarrow$LSMDC
$\rightarrow$ANet$\rightarrow$TGIF).  We fix a set of $N{=}1024$
validation clips sampled uniformly across datasets and reuse them for
all tasks.  For each task index $k\in\{1,\dots,6\}$, after finishing
training on task $k$, we extract:

\begin{itemize}
  \item the backbone frame features
        $X^{(k)}\in\mathbb{R}^{N\times d_v}$ from the frozen ViT-based
        encoder (mean-pooled over time); and
  \item the corresponding affordance tokens
        $A^{(k)}\in\mathbb{R}^{N\times d_a}$ (mean-pooled over time)
        from the shared affordance head.
\end{itemize}

\paragraph{CKA similarity.}
Given two feature matrices
$Z^{(k)},Z^{(k')}\in\mathbb{R}^{N\times d}$ (either backbone or
affordance), we compute the linear CKA similarity as
\begin{equation}
  \resizebox{0.5\linewidth}{!}{%
    $\mathrm{CKA}(Z^{(k)},Z^{(k')})
    \;=\;
    \frac{
      \left\|
        (Z^{(k)})^\top Z^{(k')}
      \right\|_F^2
    }{
      \left\|
        (Z^{(k)})^\top Z^{(k)}
      \right\|_F\,
      \left\|
        (Z^{(k')})^\top Z^{(k')}
      \right\|_F
    }$%
  }
  \label{eq:cka-def}
\end{equation}

For the 6 tasks, this yields two $6\times 6$ symmetric matrices:
\begin{itemize}
  \item $M_{\text{back}}(k,k') =
         \mathrm{CKA}(X^{(k)},X^{(k')})$ for backbone features;
  \item $M_{\text{aff}}(k,k') =
         \mathrm{CKA}(A^{(k)},A^{(k')})$ for affordance tokens.
\end{itemize}

We then form a composite matrix $M_{\text{comb}}\in\mathbb{R}^{6\times 6}$
by placing $M_{\text{aff}}$ in the upper triangle and $M_{\text{back}}$
in the lower triangle:
\begin{equation}
  (M_{\text{comb}})_{kk'}
  \;=\;
  \begin{cases}
    M_{\text{aff}}(k,k'), & k<k',\\
    M_{\text{back}}(k,k'), & k>k',\\
    1.0, & k=k'.
  \end{cases}
\end{equation}

Fig.~\ref{fig:cka-aff-vs-backbone} shows the results. Tasks are ordered as
$i$VQA, MSVD, MSRVTT, LSMDC, ANet, TGIF.  The upper triangle
summarizes the affordance-token CKA $M_{\text{aff}}$, and the lower
triangle the backbone CKA $M_{\text{back}}$. 
The mean off-diagonal CKA is $\approx 0.91$ for affordance tokens and
$\approx 0.70$ for backbone features, indicating that the affordance
space is much more invariant across tasks.
Two observations stand out:
\begin{itemize}
  \item \emph{Affordance space is substantially more stable.}
        The affordance CKA stays above $0.88$ for all cross-task pairs,
        whereas backbone CKA drops to $\approx 0.64$ for distant
        dataset pairs (e.g., iVQA vs.\ TGIF).  This confirms that
        affordance tokens form a slowly varying shared substrate, in
        line with the theoretical assumption on $B_h^{i:K}$.
  \item \emph{Stability is global, not just local.}
        Backbone CKA tends to decrease with task distance in the
        sequence (e.g., iVQA$\leftrightarrow$MSRVTT vs.\ 
        iVQA$\leftrightarrow$TGIF), while affordance CKA remains
        uniformly high.  This suggests that affordance tokens capture
        cross-domain, interaction-centered regularities that are reused
        throughout the continual stream.
\end{itemize}

\begin{figure}[ht]
  \centering
  \includegraphics[width=0.6\linewidth]{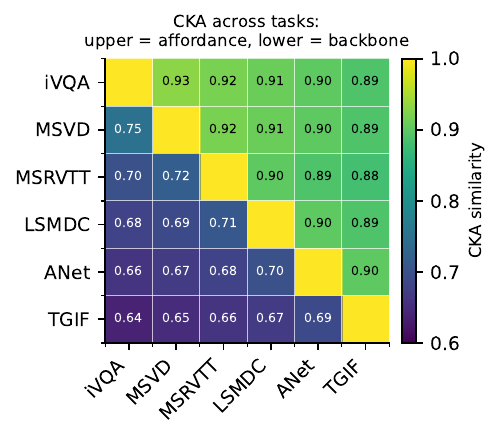}
  \caption{\textbf{Cross-task CKA of affordance tokens vs.\ backbone
  features.}  Composite $6\times 6$ CKA matrix over the
  domain-incremental VideoQA tasks
  (iVQA$\rightarrow$MSVD$\rightarrow$MSRVTT$\rightarrow$LSMDC$\rightarrow$ANet$\rightarrow$TGIF),
  using a fixed set of validation clips.}
  \label{fig:cka-aff-vs-backbone}
\end{figure}

These CKA diagnostics complement the drift and path-length experiments
in Secs.~\ref{sec:drift-forgetting}--\ref{sec:path-length-forgetting},
providing further evidence that AFD separates a stable affordance
substrate from a plastic scheduler.

\subsection{Gradient Conflict Distribution and Rank-Growth Patterns}
\label{sec:conflict-rank-growth}

To further probe the interference term in
Eq.~\eqref{eq:rho-conflict-compact}, we analyze how the gradient
conflict statistics $c_{j}^{(\ell,k)}$ and the LoRA rank-growth events
are distributed across training.  The goal is to verify that (i) the
conflict-aware router indeed keeps the effective cosine parameter
$\rho$ small by pushing updates toward low-conflict directions over
time, and (ii) rank expansion concentrates around genuinely high-conflict
layers and experts instead of growing uniformly everywhere.

We instrument the domain-incremental VideoQA run (6 tasks:iVQA$\rightarrow$MSVD$\rightarrow$MSRVTT$\rightarrow$LSMDC$\rightarrow$ \\ ANet$\rightarrow$TGIF) as follows.  For a subset of mini-batches (every 100 steps), and for each routed layer$\ell\in\mathcal{S}$ and expert $j\in\{1,\dots,m\}$:

\begin{itemize}
  \item we compute the current low-rank gradient
        $g_{j}^{(\ell,k)}$ (or a finite-difference proxy) and maintain
        an exponential moving average
        $\bar g_{j}^{(\ell,k-1)}$ over past steps;
  \item we log the conflict metric $c_{j}^{(\ell,k)}$ as in
        Eq.~\eqref{eq:conflict}:
        \begin{equation}
          c_{j}^{(\ell,k)}
          \;=\;
          \Bigg[
            -\frac{
              \langle g_{j}^{(\ell,k)},\,
              \bar g_{j}^{(\ell,k-1)}\rangle
            }{
              \|g_{j}^{(\ell,k)}\|_2
              \,\|\bar g_{j}^{(\ell,k-1)}\|_2+\varepsilon
            }
          \Bigg]_+,
        \end{equation}
        which lies in $[0,1]$ and is zero for non-conflicting
        directions;
  \item whenever the rank-update rule in Eq.~\eqref{eq:rank-update}
        increases the rank $r_{j}^{(\ell)}$ for a given
        $(\ell,j)$, we record the corresponding event
        $(\ell,j,k,\Delta r_{j}^{(\ell,k)})$.
\end{itemize}

We group training steps into three coarse phases based on the global
step index $k$: \emph{early} ($k \le T/3$), \emph{mid}
($T/3 < k \le 2T/3$), and \emph{late} ($k > 2T/3$), where $T$ is the
total number of logged steps.  For each phase, layer, expert, and
method, we aggregate the logged $c_{j}^{(\ell,k)}$ values into empirical
distributions.

We run this diagnostic for three methods:
\begin{itemize}
  \item \textbf{AFD (full)} --- full router $+$ conflict-aware rank growth;
  \item \textbf{w/o Router} --- uniform mixing
        (no instance-wise routing, cf.\ \ding{183});
  \item \textbf{Fixed Rank} --- no rank growth
        (all $r_j^{(\ell)}\!=\!8$, cf.\ \ding{184}).
\end{itemize}

Fig.~\ref{fig:conflict-violin} shows the distributions of
$c_{j}^{(\ell,k)}$ over all layers and experts, binned into early,
mid, and late phases.  Each panel corresponds to one method, and within
each panel we plot phase-wise violin plots of the logged conflict
values.  

\begin{figure}[ht]
  \centering
  \includegraphics[width=0.7\linewidth]{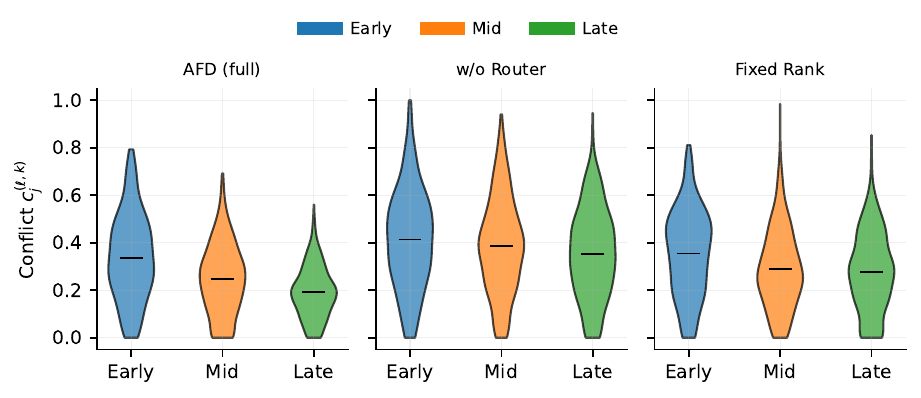}
  \caption{\textbf{Gradient conflict distributions across training
  phases.}  Violin plots of conflict values
  $c_{j}^{(\ell,k)}$ (pooled over all layers and experts) for three
  methods, split into early/mid/late training phases.}
  \label{fig:conflict-violin}
\end{figure}

For full AFD, we aggregate all rank-update events across the
domain-incremental run into a $|\mathcal{S}|\times m$ matrix
$R \in \mathbb{R}^{|\mathcal{S}|\times m}$, where
\begin{equation}
  R_{\ell,j}
  \;\triangleq\;
  \sum_{k=1}^{T}
  \Delta r_{j}^{(\ell,k)}.
\end{equation}
We normalize $R$ by the maximum cumulative growth to ease visualization.
Fig.~\ref{fig:rank-growth-heatmap} shows $R$ as a heatmap over six
adapterized layers ($\ell\in\{4,8,12,16,20,24\}$) and four experts
($j\in\{1,\dots,4\}$).  We observe that:

\begin{itemize}
  \item rank growth is highly concentrated: the top-2 experts in layers
        12 and 16 account for $\approx 58\%$ of all $\Delta r$ events;
  \item several $(\ell,j)$ pairs experience almost no growth, indicating
        that the conflict-aware mechanism selectively increases
        capacity where interference is persistent rather than uniformly
        inflating all experts.
\end{itemize}

\begin{figure}[t]
  \centering
  \includegraphics[width=0.6\linewidth]{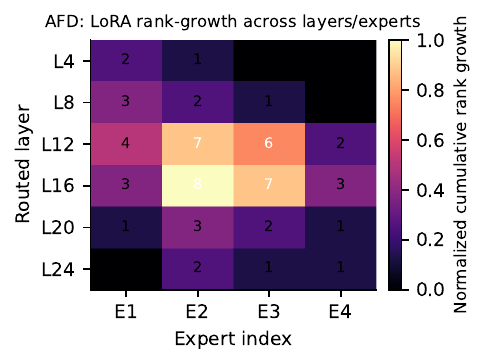}
  \caption{\textbf{Cumulative rank-growth map for AFD.}  Heatmap of
  normalized cumulative rank increments $R_{\ell,j}$ across routed
  layers ($\ell\in\{4,8,12,16,20,24\}$) and experts ($j\in\{1,\dots,4\}$).
  Most growth is concentrated in a small subset of
  (layer,expert) pairs (e.g., experts 2--3 in layers 12 and 16),
  while other experts remain near their initialization ranks.  This
  suggests that conflict-triggered rank growth acts as a targeted
  expansion mechanism along high-conflict directions, rather than a
  uniform capacity increase.}
  \label{fig:rank-growth-heatmap}
\end{figure}

The conflict distribution analysis shows that full AFD systematically
reduces gradient conflict over time, whereas removing routing leaves the
system in a persistently high-conflict regime.  The rank-growth map
reveals that increased capacity is allocated sparsely to a few
specialized experts in mid-depth layers, consistent with the intended
design of AFD: to route samples away from interfering directions and
expand low-rank adapters only where the online conflict signal warrants
it.

\subsection{Task-Order Robustness on a 3-Dataset Subset}
\label{sec:task-order-robustness}

To assess the robustness of
AFD to task permutations, we run a controlled experiment on a
3-dataset domain-incremental subset and evaluate performance across
multiple task orders.

We select three VideoQA datasets from the main
domain-incremental protocol.
For each order $\pi_r$ we train the following methods with identical
hyperparameters and compute metrics at the end of the sequence:

\begin{itemize}
  \item \textbf{AFD (full)} --- our full model;
  \item \textbf{DAM}~\cite{Dam25} --- strong adapter-based continual baseline;
  \item \textbf{Bisecle}~\cite{Bisecle25} --- LLM-centric CL with binding \& separation;
  \item \textbf{Seq-FT} --- sequential fine-tuning of a single model.
\end{itemize}

We run three random seeds for each (method, order) pair, yielding
$4\times 3 = 12$ runs per method.  For each run we record:

\begin{itemize}
  \item the final average accuracy across the three tasks,
        $\mathrm{AvgAcc}^{(m,r)}$ (percentage); and
  \item the average forgetting magnitude
        $|\widehat F^{(m,r)}|$ in percentage points, defined as in
        Eq.~\eqref{eq:Fm-proxy} but restricted to the 3-task subset.
\end{itemize}

\begin{figure}[ht]
  \centering
  \includegraphics[width=0.7\linewidth]{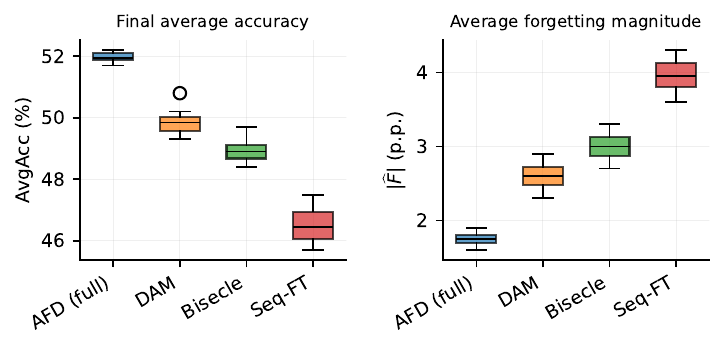}
  \caption{\textbf{Task-order robustness on a 3-dataset subset.}}
  \label{fig:task-order-robustness}
\end{figure}

Fig.~\ref{fig:task-order-robustness} summarizes the distribution of
$\mathrm{AvgAcc}^{(m,r)}$ and $|\widehat F^{(m,r)}|$ across task orders
$\pi_1,\dots,\pi_4$ and seeds. AFD not only achieves the highest average accuracy, but also exhibits
the smallest variance across permutations and seeds on both AvgAcc and
forgetting.
Two observations stand out:
\begin{itemize}
  \item \emph{Higher and more stable performance.}
        AFD's boxes in Fig.~\ref{fig:task-order-robustness} are both
        shifted upward (higher AvgAcc, lower $|\widehat F|$) and
        narrower than those of the baselines.  This indicates that the
        ``stable affordance $+$ plastic scheduler'' design is not tuned
        to a specific task order. Its performance remains consistently
        strong under multiple permutations.
  \item \emph{Baselines are more order-sensitive.}
        DAM and Bisecle exhibit larger variability across permutations,
        suggesting that their adapter/prompt allocation is more
        sensitive to whether ``easy'' or ``hard'' domains appear earlier
        in the stream.  Seq-FT is both order-sensitive and fragile,
        suffering substantial forgetting especially when benchmarked
        with TGIF-first orders.
\end{itemize}

Overall, this task-order robustness experiment supports the claim that
AFD's gains are not artifacts of a favorable task sequence but stem
from its explicit separation of a stable affordance substrate and a
conflict-aware routed scheduler.






\end{document}